\documentclass[12pt]{article}
\usepackage{setspace}
\usepackage{amsfonts}
\usepackage{srcltx}
\usepackage{url}
\usepackage{multirow}
\usepackage{tabularx}
\usepackage{array,makecell}
\usepackage{rotating}
\usepackage[caption = true]{subfig}
\usepackage{graphicx}
\usepackage{epstopdf}
\usepackage{caption}
\usepackage{amsmath}
\usepackage{amsthm}
\usepackage{amssymb}
\usepackage{titlesec}
\usepackage[authoryear]{natbib}
\usepackage{microtype}
\usepackage{enumerate}
\usepackage{hyperref}
\usepackage{xcolor}
\usepackage{here}
\usepackage{authblk}
\makeatletter
\numberwithin{equation}{section}
\newtheorem{theorem}{Theorem}[section]
\newtheorem{lemma}{Lemma}[section]

\newtheorem{corollary}[theorem]{Corollary}
\theoremstyle{definition}
\newtheorem{definition}{Definition}[section]

\theoremstyle{remark}

\makeatletter

\newcommand{\Rmnum}[1]{\expandafter\@slowromancap\romannumeral #1@}
\makeatother

\newcommand{\blind}{1}

\usepackage{threeparttable} 
\newtheorem{assumption}{Assumption}[section]
\newcommand{\eins}{\boldsymbol{1}}

\usepackage[ruled]{algorithm2e}
\usepackage{booktabs}

\begin{document}
	%\begin{doublespace}
	%\begin{singlespace}
	\def\spacingset#1{\renewcommand{\baselinestretch}%
		{#1}\small\normalsize} \spacingset{1}

	%%%%%%%%%%%%%%%%%%%%%%%%%%%%%%%%%%%%%%%%%%%%%%%%%%%%%%%%%%%%%%%%%%%%%%%%%%%%%%
	
	\if1\blind
	{
		\title{\bf Density-based Clustering with \\ Best-scored Random Forest}
		\author{Hanyuan Hang$^{\dagger}$, Yuchao Cai$^{\ddagger}$, and Hanfang Yang$^{\ddagger}$\footnote{Corresponding author. E-Mail: hyang@ruc.edu.cn.}
			\\
			${}^{\dagger}$Institute of Statistics and Big Data,
			Renmin University of China 
			\and
			${}^{\ddagger}$School of Statistics,
			Renmin University of China}
		\maketitle
	} \fi
	
	\if0\blind
	{
		\bigskip
		\bigskip
		\bigskip
		\begin{center}
			{\LARGE\bf Adaptive Density-based Clustering with \\ Best-scored Random Forest}
		\end{center}
		\medskip
	} \fi
	
	\bigskip
	\begin{abstract}Single-level density-based approach has long been widely acknowledged to be a conceptually and mathematically convincing clustering method. 
		In this paper, we propose an algorithm called \emph{best-scored clustering forest}
		that can obtain the optimal level and determine corresponding clusters.
		The terminology \emph{best-scored} means to select one random tree with the best empirical performance out of a certain number of purely random tree candidates.
		From the theoretical perspective, we first show that consistency of our proposed algorithm can be guaranteed. 
		Moreover, under certain mild restrictions on the underlying density functions and target clusters, even fast convergence rates can be achieved.
		Last but not least, comparisons with other state-of-the-art clustering methods in the numerical experiments
		demonstrate accuracy of our algorithm on both synthetic data and several benchmark real data sets.
	\end{abstract}

	\noindent%
	\vfill
	\newpage
	\spacingset{1.45}
	
%	\newpage
%	
%	\tableofcontents
%	
%	\newpage

	\section{Introduction} \label{sec::Introduction}

	Regarded as one of the most basic tools to investigate statistical properties of unsupervised data, clustering aims to group a set of objects in such a way that objects in the same cluster are more similar in some sense to each other than to those in other clusters. 
	Typical application possibilities are to be found reaching from categorization of tissues in medical imaging to grouping internet searching results. For instance, on PET scans, cluster analysis can distinguish between different types of tissue in a three-dimensional image for many different purposes \citep{FILIPOVYCH20112185} while in the process of intelligent grouping of the files and websites, clustering algorithms create a more relevant set of search results \citep{DiMarco2013}. 
	Because of their wide applications, more urgent requirements for clustering algorithms that not only maintain desirable prediction accuracy but also have high computational efficiency are raised. In the literature, a wealth of algorithms have already been proposed
	such as $k$-means \citep{Macqueen67somemethods}, linkage \citep{JoeH1963Hierarchical, Sibson1973SLINK, Defays1977An}, cluster tree \citep{Stuetzle2003Estimating}, DBSCAN \citep{Ester1996A}, spectral clustering \citep{Donath1973Lower, Luxburg2007A}, and expectation-maximization for generative models \citep{Dempster1977Maximum}. 
	
	As is widely acknowledged, an open problem in cluster analysis is how to describe a conceptually and mathematically convincing definition of clusters appropriately. 
	In the literature, great efforts have been made to deal with this problem. 
	Perhaps the first definition dates back to \cite{Har1975}, which is known as the single-level density-based clustering assuming i.i.d.~data $D = (x_1, \dots, x_n)$ generated by some unknown distribution $\mathrm{P}$ that has a continuous density $h$ and the clusters of $\mathrm{P}$ are then defined to be the connected components of the level set $\{ h \geq \rho \}$ given some $\rho \geq 0$. Since then, different methods based on the estimator $\hat{h}$ and the connected components of \{$h \geq \rho\}$ have been established \citep{Cuevas1997A, Maier2012Optimal, Rigollet2006Generalization, Rinaldo2010GENERALIZED}.

	Note that the single-level approach mentioned above is easily shown to have a conceptual drawback that different values of $\rho$ may lead to different (numbers of) clusters, and there is also no general rule for choosing $\rho$. 
	In order to address this conceptual shortcoming, another type of the clustering algorithms, namely hierarchical clustering, where the hierarchical tree structure of the connected components for different levels $\rho$ is estimated, was proposed. Within this framework, instead of choosing some $\rho$, the so-called cluster tree approach tries to consider all levels and the corresponding connected components simultaneously. 
	It is worth pointing out that the advantage of using cluster tree approach lies in the fact that it mainly focuses on the identification of the hierarchical tree structure of the connected components for different levels. 
	For this reason, in the literature, there have already been many attempts to establish their theoretical foundations.
	For example, \cite{Hartigan1981Consistency} proved the consistency of 
	a hierarchical clustering method named single linkage merely for the one-dimensional case
	which becomes a more delicate problem that it is only fractionally consistent in the high-dimensional case. To address this problem, \cite{Chaudhuri10ratesof} proposed a modified single linkage algorithm which is shown to have finite-sample convergence rates as well as lower bounds on the sample complexity under certain assumptions on $h$. 
	Furthermore, \cite{Kpotufe_pruningnearest} obtained similar theoretical results with an underlying $k$-NN density estimator and achieved experimental improvement by means of
	a simple pruning strategy that removes connected components that artificially occur because of finite sample variability. 
	However, the notion of \emph{recovery} taken from \cite{Hartigan1981Consistency} falls short of only focusing on the correct estimation of the cluster tree structure and not on the estimation of the clusters itself, more details we refer to \cite{Rinaldo2010GENERALIZED}.

	So far, the theoretical foundations for hierarchical clustering algorithms such as consistency and learning rates of the existing hierarchical clustering algorithms are only valid for the cluster tree structure and therefore far from being satisfactory. As a result, in this paper, we proceed with the study of single-level density-based clustering. In the literature, recently, 
	various results for estimating the optimal level have already been established. 
	First of all, \cite{steinwart2011adaptive} and \cite{steinwart2015adaptive} presented algorithms based on histogram density estimators that are able to asymptotically determine the optimal level and automatically yield a consistent estimator for the target clusters.  Obviously, these algorithms are of little practical value since only the simplest possible density estimators are considered. Attempting to address this issue, \cite{sriperumbudur2012consistency} proposed a modification of the popular DBSCAN clustering algorithm. 
	Although consistency and optimal learning rates have been established
	for this new DBSCAN-type construction,
    the main difficulty in carrying out this algorithm
     is that it restricts the consideration only to moving window density estimators for $\alpha$-H\"{o}lder continuous densities. 
	In addition, it's worth 
    noticing that
	none of the algorithms mentioned above can be well adapted to the case where the underlying distribution possesses no split in the cluster tree.
	To tackle this problem, \cite{steinwart2017adaptive} proposed an adaptive algorithm using kernel density estimators which, however, also only performs well for low-dimensional data.

	In this paper, we mainly focus on clusters that are defined as the connected components of high density regions
	and present an algorithm called \emph{best-scored clustering forest} which can not only guarantee consistency and attain fast convergence rates, but also enjoy satisfactory performance in various numerical experiments. 
	To notify, the main contributions of this paper are twofold: 
	\emph{(i)} Concerning with the theoretical analysis, we prove that with the help of the best-scored random forest density estimator, our proposed algorithm can ensure consistency and achieve fast convergence rates under certain assumptions for the underlying density functions and target clusters. We mention that the convergence analysis is conducted within the framework established in \cite{steinwart2015adaptive}. To be more precise, 
	under properly chosen hyperparameters of 
	the best-scored random forest density estimator \cite{hang2018best},
     the consistency 
     of the best-scored clustering forest
     can be ensured. Moreover, under some additional regularization conditions, even fast convergence rates can be achieved.
	\emph{(ii)} When it comes to numerical experiments, we improve the original purely random splitting criterion by proposing an adaptive splitting method. Instead, at each step, we randomly select a sample point from the training data set and the to-be-split node is the one which this point falls in. The idea behind this procedure is that when randomly picking sample points from the whole training data set, nodes with more samples will be more likely to be chosen whereas nodes containing fewer samples are less possible to be selected. In this way, 
	the probability to obtain cells with sample sizes evenly distributed will be much greater. Empirical experiments further show that the adaptive/recursive method enhances the efficiency of our algorithm since it actually increases the \emph{effective} number of splits.

	The rest of this paper is organized as follows: Section \ref{sec::Preliminaries} introduces some fundamental notations and definitions related to the density level sets and best-scored random forest density estimator. 
	Section \ref{sec::AGenericClusteringAlgorithmBasedonBRDF} is dedicated to the exposition of the generic clustering algorithm architecture.
	We provide our main theoretical results and statements on the consistency and learning rates of the proposed best-scored clustering forest in Section \ref{sec::ConsistencyandRatesforRFDE-basedClustering},
	where the main analysis aims to verify that our best-scored random forest could provide level set estimator that has control over both its vertical and horizontal uncertainty. 
	Some comments and discussions on the established theoretical results will be also presented in this section. 
	Numerical experiments conducted upon comparisons between best-scored clustering forest and other density-based clustering methods are given in Section \ref{sec::ExperimentalPerformance}. 
	All the proofs of Section \ref{sec::AGenericClusteringAlgorithmBasedonBRDF} and  Section \ref{sec::ConsistencyandRatesforRFDE-basedClustering} can be found in Section \ref{sec::Proofs}. We conclude this paper with a brief discussion in the last section.
	
%	In the last section, we present a brief summary and discussion of this paper.
%	Finally, we mention that all the proofs of Sections \ref{sec::AGenericClusteringAlgorithmBasedonBRDF} and \ref{sec::ConsistencyandRatesforRFDE-basedClustering} can be found in the supplement material.

	\section{Preliminaries}    \label{sec::Preliminaries}

	In this section, we recall several basic concepts and notations related to clusters in the first subsection while in the second subsection we briefly recall the best-scored random forest density estimation proposed recently by \cite{hang2018best}. 
	
	%and show two fundamental and essential propositions for the new generic algorithm. 

	\subsection{Density Level Sets and Clusters} \label{subsec::DensityLevelSetsandClusters}

	This subsection begins by introducing some basic notations and assumptions about density level sets and clusters. 
	Throughout this paper, 
	let $\mathcal{X} \subset \mathbb{R}^d$ be a compact and connected subset, $\mu := \lambda^d$ be the Lebesgue measure with $\mu(\mathcal{X}) > 0$. 
	Moreover, let $\mathrm{P}$ be a probability measure that is absolutely continuous with respect to $\mu$ 
	and possess a bounded density $f$ with support $\mathcal{X}$. 
	We denote the centered hypercube of $\mathbb{R}^d$ with side length $2r$ by $B_r$ where
	\begin{align*}
		B_r := \{ x = (x_1, \ldots, x_d) \in \mathbb{R}^d : x_i \in [-r, r], i = 1, \ldots, d \},
	\end{align*}
	and the complement of $B_r$ is written by $B_r^c := \mathbb{R}^d \setminus [-r, r]^d$.

	Given a set $A \subset \mathcal{X}$, we denote by $\mathring{A}$ its interior, $\bar{A}$ its closure, $\partial{A} = \bar{A} \setminus \mathring{A}$ its boundary, and 
	$\mathrm{diam}(A) := \sup_{x, x' \in A}  \|x - x'\|_2$ its diameter. 
	Furthermore, for a given $x$, $d(x,A)  :=  \inf_{x' \in A} \|x-x'\|_2$ denotes the distance between $x$ and $A$. 
	Given another set $B \subset \mathcal{X}$, we denote by
	$A \triangle B$ the symmetric difference between $A$ and $B$.
	Moreover, $\eins_A$ stands for the indicator function of the set $A$.

	We say that a function $f : \mathbb{R}^d \to \mathbb{R}$ is $\alpha$-H\"{o}lder continuous, if there exists a constant $c > 0$ such that
	\begin{align*}
		|f(x)-f(x')| \leq c \|x - x'\|_2^{\alpha},
		\qquad \qquad
		\alpha \in (0, 1].
	\end{align*}
	To mention, it can be apparently seen that $f$ is constant whenever $\alpha > 1$. 
	
	Finally, throughout this paper, we use the notation $a_n \lesssim b_n$ to denote that there exists a positive constant $c$ such that $a_n \leq c b_n$, for all $n \in \mathbb{N}$.
	
	\subsubsection{Density Level Sets} \label{subsec::DensityLevelSets}

	In order to find a notion of density level set which is topologically invariant against different choices of the density $f$ of the distribution $\mathrm{P}$, \cite{steinwart2011adaptive} proposes to define a density level set at level $\rho \geq 0$ by
	\begin{align*}
		M_{\rho} := \mathrm{supp} \, \mu_{\rho}
	\end{align*}
	where $\mathrm{supp} \, \mu_{\rho}$ stands for the support of $\mu_{\rho}$, and the measure $\mu_{\rho}$  is defined by 
	\begin{align*}
		\mu_{\rho}(A) := \mu(A \cap \{ f \geq \rho \}),
		\qquad \qquad  
		A \in \mathcal{B}(\mathcal{X}),
	\end{align*}
	where $\mathcal{B}(\mathcal{X})$ denotes the Borel $\sigma$-algebra of $\mathcal{X}$.
	According to the definition, the density level set $M_\rho$ should be closed. 
	If the density $f$ is assumed to be $\alpha$-H\"{o}lder continuous, the above construction could be replaced by the usual $\{f\geq\rho \}$ without changing our results. 

    \begin{figure*}[htbp]
    	\begin{minipage}[t]{0.99\textwidth}  
    		\centering  
    		\includegraphics[width=\textwidth]{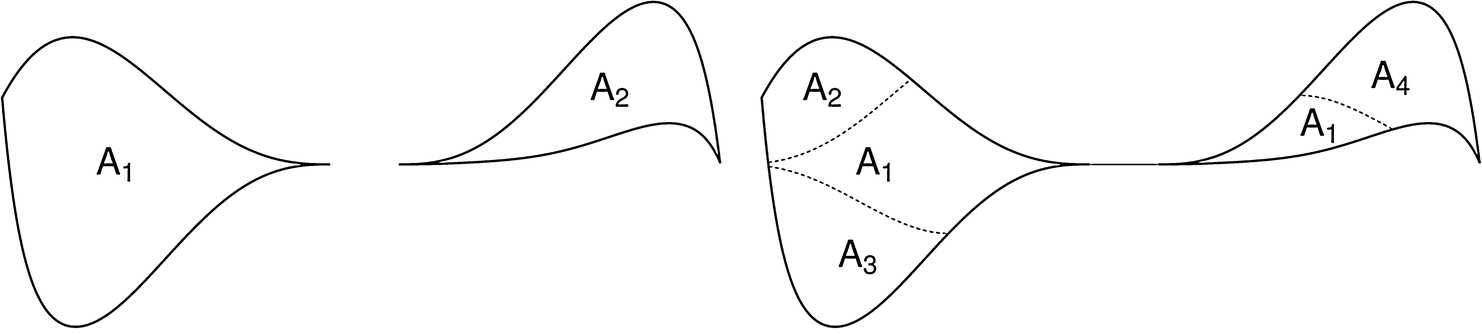}  
    	\end{minipage}  
    	\centering  
    	\caption{Topologically relevant changes on set of measure zero. 
    		Left: The thick solid lines indicate a set consisting of two connected components $A_1$ and $A_2$. The density of $\mathrm{P}$ is $f=c \eins_{A_1 \cup A_2}$ with $c$ being a suitable constant, then $A_1$ and $A_2$ are the two connected components of $\{ f \geq \rho\}$ for all $\rho \in [0, c]$. 
    		Right: This is a similar situation. The straight horizontal thin line indicates a line of measure zero connecting the two components, and the dashed lines indicate cuts of measure zero. In this case, the density of $\mathrm{P}$ is $f' = c \eins_{A_1 \cup A_2 \cup A_3 \cup A_4}$, then $A_1$, $A_2$, $A_3$, and $A_4$ are the four connected components of $\{ f' \geq \rho \}$ for all $\rho \in [0, c]$.}
    	\label{fig:deflevelset}
    \end{figure*}

	Here, some important properties of the sets $M_\rho$, $\rho\geq 0 $ are useful:
	\begin{enumerate}[(i)]
		\item 
	\emph{Level Sets}. 
	$\overline{\{f>\rho\}} \subset M_\rho \subset\{ f \geq\rho \}$
		\item 
	\emph{Monotonicity}. 
	$M_{\rho_2}\subset M_{\rho_1}$ for all $\rho_1\leq \rho_2$.
		\item 
	\emph{Regularity}. 
	$\mu (M_\rho \triangle \{ f \geq\rho\})=0$.
		\item 
	\emph{Normality}. 
	$\bar{M}_\rho=\dot{M}_\rho$, where $\bar{M}_\rho:=\bigcup_{\rho'>\rho}\, M_{\rho'}$ and $\dot{M}_\rho:=\bigcup_{\rho'>\rho}\, \mathring{M}_{\rho'}$.
		\item 
	\emph{Open Level Sets}. 
	$\bar{M}_\rho=\{ f >\rho\}$.
	\end{enumerate}

	\subsubsection{Comparison of Partitions and Notations of Connectivity} \label{subsubsec::ComparisonofPartitionsandNotationsofConnectivity}

	Before introducing the definition of clusters, some notions related to the connected components of level sets are in need. First of all, we give the definition that compares different partitions.

	\begin{definition}	\label{def::partition}
		Let $A, B \subset \mathcal{X}$ be nonempty sets with $A \subset B$,
		and $\mathcal{P}(A)$ and $\mathcal{P}(B)$ be partitions of $A$ and $B$, respectively. 
		Then $\mathcal{P}(A)$ is said to be \emph{comparable} to $\mathcal{P}(B)$, if for all $A' \in \mathcal{P}(A)$, there exists a $B' \in \mathcal{P}(B)$ such that $A'\subset B'$.
		In this case, we write $\mathcal{P}(A) \sqsubset \mathcal{P}(B)$.
	\end{definition}
	
	It can be easily deduced that $\mathcal{P}(A)$ is comparable to $\mathcal{P}(B)$, if no cell $A'\in \mathcal{P}(A)$ is broken into pieces in $\mathcal{P}(B)$. 
	Let $\mathcal{P}_1$ and $\mathcal{P}_2$ be two partitions of $A$, then we call $\mathcal{P}_1$ is \emph{finer} than $\mathcal{P}_2$ if and only if $\mathcal{P}_1\sqsubset \mathcal{P}_2$. Moreover, as is demonstrated in \cite{steinwart2015suppA}, for two partitions $\mathcal{P}(A)$ and $\mathcal{P}(B)$ with $\mathcal{P}(A)\sqsubset\mathcal{P}(B)$, there exits a unique map $\zeta:\mathcal{P}(A)\to \mathcal{P}(B)$ such that $A'\subset \zeta(A')$ for $A'\in \mathcal{P}(A)$. We call $\zeta$ the \emph{cell relating map (CRM)} between $A$ and $B$.
	
	Now, we give further insight into two vital examples of comparable partitions coming from connected components. Recall that an $A\subset \mathcal{X}$ is \emph{topologically connected} if, for every pair $A',A''\subset A$ of relatively closed disjoint subsets of $A$ with $A'\cup A''=A$, we have $A'=\emptyset$ or $A''=\emptyset$. The maximal connected subsets of $A$ are called the \emph{connected components} of $A$. As is widely acknowledged, these components make up a partition of $A$, and we denote it by $\mathcal{C}(A)$. Furthermore, for a closed $A\subset B$ with $|\mathcal{C}(B)|<\infty$, we have $\mathcal{C}(A) \sqsubset \mathcal{C}(B)$.
	
	The next example describes another type of connectivity, namely \emph{$\tau$-connectivity},
	which can be considered as
	 a discrete version of \emph{path-connectivity}. For the latter, let us fix a $\tau>0$ and $A\in \mathcal{X}$.
%	and a norm $\|\cdot\|$ on $\mathbb{R}^d$.
	 Then, $x,x'\in A$ are called \emph{$\tau$-connected} in $A$, if there exists $x_1,\ldots,x_n\in A$ such that $x_1=x$, $x_n=x'$ and $\|x_i-x_{i+1}\|_2<\tau$ for all $i=1,\ldots,n-1$. 
	 Clearly, being $\tau$-connected gives an equivalence relation on $A$. To be specific, the resulting partition can be written as $\mathcal{C}_{\tau}(A)$, and we call its cells the \emph{$\tau$-connected components} of $A$. 
	 It can be verified that, for all $A\subset B$ and $\tau>0$, we always have $C_\tau(A) \sqsubset C_\tau(B)$, see
	 Lemma A.2.7 in \cite{steinwart2015suppA}. In addition, if $|\mathcal{C}(A)|<\infty$, then we have $\mathcal{C}(A) = \mathcal{C}_\tau(A)$ for all sufficiently small $\tau > 0$, see Section 2.2 in \cite{steinwart2015adaptive}.

	\subsubsection{Clusters} \label{subsubsec::Clusters}

	Based on the concept established in the preceding subsections we now recall the definition of clusters, see also Definition 2.5 in \cite{steinwart2015suppA}.
	
	\begin{definition}[Clusters] \label{def::ClusterProperty}
		Let $\mathcal{X}\subset \mathbb{R}^d$ be a compact and connected set, and $\mathrm{P}$ be a $\mu$-absolutely continuous distribution. Then $\mathrm{P}$ \emph{can be clustered} between $\rho^*\geq 0$ and $\rho^{**}>\rho^*$, if $\mathrm{P}$ is normal and for all $\rho\in [0,\rho^{**}]$, the following three conditions are satisfied: 
		\begin{enumerate}[(i)]
			\item We have either $|\mathcal{C}(M_\rho)| = 1$ or $|\mathcal{C} (M_\rho)| = 2$;
			\item If we have $|\mathcal{C}(M_\rho)|=1$, then $\rho\leq \rho^*$;
			\item If we have $|\mathcal{C}(M_\rho)|=2$, then $\rho\geq \rho^*$ and $\mathcal{C}(M_{\rho^{**}})\sqsubset \mathcal{C}(M_\rho)$.
		\end{enumerate}
		Using the CRMs $\zeta_{\rho}:\mathcal{C}(M_{\rho^{**}})\to \mathcal{C}(M_{\rho})$; we then define the \emph{clusters} of $\mathrm{P}$ by
		\begin{align*}
			A_i^*=\bigcup_{\rho\in (\rho^*,\rho^{**}]} \zeta_\rho (A_i),\quad i\in\{1,2\},
		\end{align*}
		where $A_1$ and $A_2$ are the two topologically connected components of $M_{\rho^{**}}$. Finally, we define
		\begin{align}\label{eq::tau}
			\tau^*(\varepsilon):=\frac{1}{3}d((\zeta_{\rho^*+\varepsilon}(A_1),\zeta_{\rho^*+\varepsilon}(A_2)),
			\qquad \qquad 
			\varepsilon\in (0,\rho^{**}-\rho^{*}]. 
		\end{align}
	\end{definition}
	
	\begin{figure*}[htbp]
		\begin{minipage}[t]{0.99\textwidth}  
			\centering  
			\includegraphics[width=\textwidth]{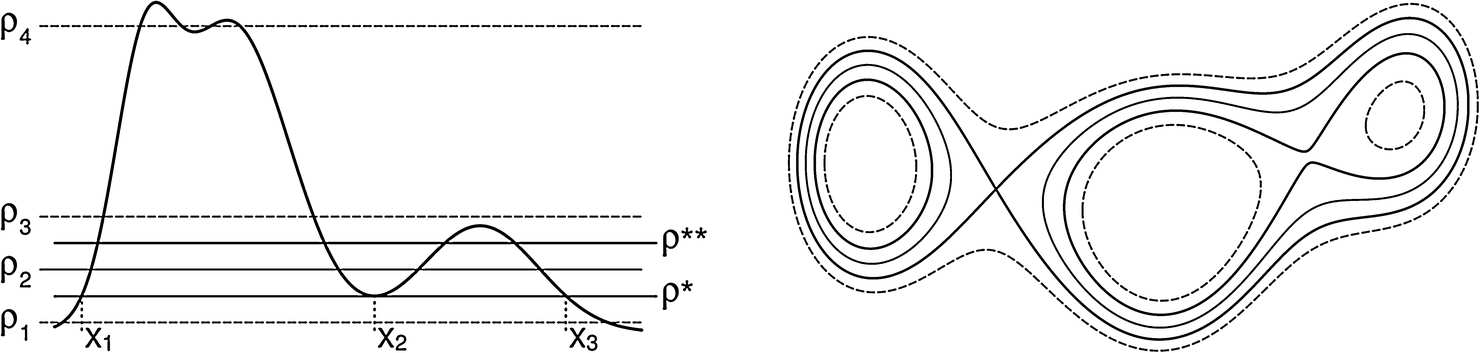}  
		\end{minipage}  
		\centering  
		\caption{Definition of clusters. 
		Left: A one-dimensional mixture of three Guassians with the optimal level $\rho^*$ and a possible choice of $\rho^{**}$. It is easily observed that the open intervals $(x_1, x_2)$ and $(x_2, x_3)$ are the two clusters of the distribution. We only have one connected component for the level $\rho_1 < \rho^*$ and the levels $\rho_3$ and $\rho_4$ are not considered in above definition. 
		Right: Here we have a similar situation for a mixture of three two-dimensional Gaussians drawn by contour lines. The thick solid lines indicate the levels $\rho^*$ and $\rho^{**}$, while the thin solid lines show a level in $(\rho^*, \rho^{**})$. The dashed lines correspond to a level $\rho \leq \rho^*$ and a level $\rho \geq \rho^{**}$. In this case, the clusters are the two connected components by the outer thick solid line.}
		\label{fig:defcluster}
	\end{figure*} 
	
	To illustrate, the above definition ensures that the level set below $\rho^*$ are connected, while there are exactly two components in the level sets for a certain range above $\rho^*$. To notify, any two level sets between this range are supposed to be comparable. As a result, the topological structure between $\rho^*$ and $\rho^{**}$ can be determined by that of $M_{\rho^{**}}$. In this manner, the connected components of $M_{\rho}$, $\rho\in (\rho^*,\rho^{**})$ can be numbered by the connected components of $M_{\rho^{**}}$. This numbering procedure can be clearly reflected from the definition of the clusters $A_i^*$ as well as that of the function $\tau^*$, which in essence measures the distance between the two connected components at level $\rho^*+\varepsilon$.

	Concerning that the quantification of uncertainty of clusters is indispensable, we need to introduce for $A\subset \mathcal{X}$, $\delta>0$, the sets
	\begin{align}
		A^{+\delta}&:=\{x\in \mathcal{X}:d(x,A)\leq \delta\},
		\nonumber\\
		A^{-\delta}&:=\mathcal{X}\setminus (\mathcal{X}\setminus A)^{+\delta}.
		\label{DeltaTubes}
	\end{align}
	In other words, $A^{+\delta}$ can be recognized as adding a $\delta$-tube to $A$, while $A^{-\delta}$ is treated as removing a $\delta$-tube from $A$. We are expected to avoid cases where the density level sets have bridges or cusps that are too thin. To be more precise, recall that for a closed $A\subset \mathbb{R}^d$, the function $\psi^*_A:(0,\infty)\to [0,\infty)$ is defined by
	\begin{align*}
		\psi^*_A(\delta):=\sup_{x\in A} d(x,A^{-\delta}), 
		\qquad \qquad 
		\delta > 0.
	\end{align*}
	Particularly, for all $\delta>0$, we have $\psi^*_A(\delta)\geq \delta$ for all $\delta>0$, and if $A^{-\delta}=\emptyset$, then $\psi^*_A(\delta)=\infty$. Consequently, according to Lemma A.4.3 in \cite{steinwart2015suppA}, for all $\delta>0$ with $A^{-\delta}\neq \emptyset$ and all $\tau>2\psi^*(\delta)$, we have
	\begin{align*}
		|\mathcal{C}_\tau(A^{-\delta})|\leq |\mathcal{C}(A)|,
	\end{align*}
	whenever $A$ is contained in some compact $\mathcal{X}\subset \mathbb{R}^d$ and $|\mathcal{C}(A)|< \infty$.
	
	With the preceding preparations, we now come to the following definition excluding bridges and cusps which are too thin.
	
	\begin{definition} \label{def::ThickLetSets}
		Let $\mathcal{X}\subset \mathbb{R}^d$ be a compact and connected set, and $\mathrm{P}$ be a $\mu$-absolutely continuous distribution  that is normal. Then we say that \emph{$\mathrm{P}$ has thick level sets of order $\gamma\in (0,1]$} up to the level $\rho^{**}>0$, if there exits constants $c_{\textit{thick}}\geq 1$ and $\delta_{\textit{thick}}\in (0,1]$ such that, for all $\delta\in (0,\delta_{\textit{thick}}]$ and $\rho\in [0,\rho^{**}]$, we have
		\begin{align*}
			\psi^*_{M_\rho}(\delta)
			\leq c_{\textit{thick}} \delta^{\gamma}.
		\end{align*}
		In this case, we call $\psi(\delta):=3c_{\textit{thick}}\delta^\gamma$ the thickness function of $\mathrm{P}$.
	\end{definition}

	\begin{figure*}[htbp]
		\begin{minipage}[t]{0.99\textwidth}  
			\centering  
			\includegraphics[width=\textwidth]{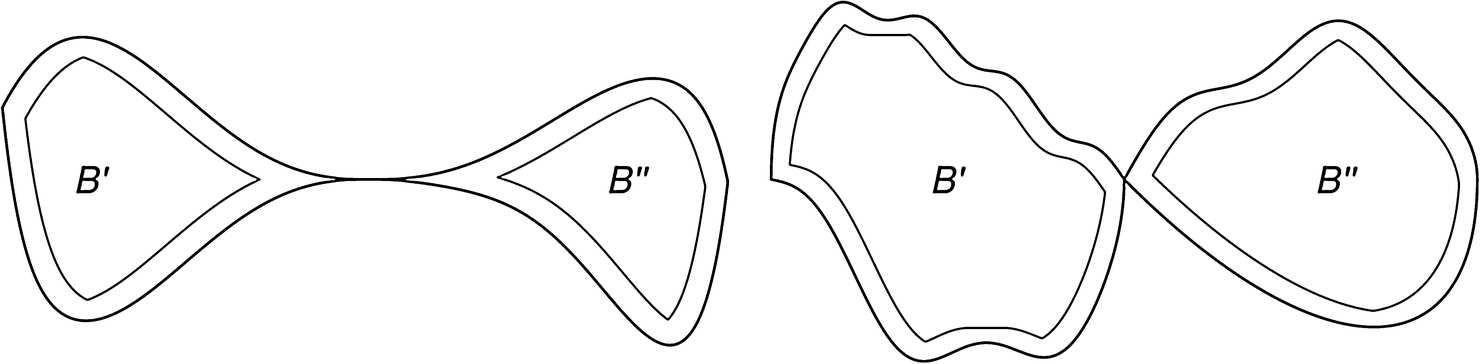}  
		\end{minipage}  
		\centering  
		\caption{Thick level sets. 
		Left: The thick solid line presents a level set $M_{\rho}$ below or at the level $\rho^*$ and the thin solid line indicates the two clusters $B'$ and $B''$ of $M_{\rho}^{-\delta}$. Since the quadratic shape of $M_{\rho}$ around the thin bridge, the distribution has thickness of order $\gamma = 1/2$. 
		Right: In the same situation, the distribution has thick level sets of order $\gamma = 1$. It is worth noting that smaller $\gamma$ leads to a significantly wider separation of $B'$ and $B''$.}
		\label{fig:thicklevelset}
	\end{figure*}

	In order to describe the distribution we wish to cluster, we now make the following assumption based on all concepts introduced so far.
	
	\begin{assumption} \label{ass::MainDistribution}
		The distribution $\mathrm{P}$ with bounded density $f$ is able to be clustered between $\rho^*$ and $\rho^{**}$. Moreover, $\mathrm{P}$ has thick level sets of order $\gamma\in (0,1]$ up to the level $\rho^{**}$. The corresponding thickness function is denoted by $\psi$ and the function defined in \eqref{eq::tau} is abbreviated as $\tau^*$.	
	\end{assumption}
	
	In the case that all level sets are connected, we introduce the following assumption
	to investigate the behavior of the algorithm in situations in which $\mathrm{P}$ cannot be clustered.
	
	\begin{assumption} \label{ass::MainDistributionaddi}
		Let $\mathcal{X}\subset \mathbb{R}^d$ be a compact and connected set, and $\mathrm{P}$ be a $\mu$-absolutely continuous distribution that is normal. 
		Assume that there exist constants $\rho_*\geq 0$, $\gamma \in (0,1]$, $c_{\textit{thick}}\geq 1$ and $\delta_{\textit{thick}}\in(0,1]$ such that for all $\rho \geq \rho_*$ and $\delta \in (0,\delta_{\textit{thick}}]$, the following conditions hold:
		\begin{itemize}
			\item[(i)] 
			$|\mathcal{C}(M_\rho)|\leq 1$.
			\item[(ii)] 
			If $M_\rho^{-\delta}\neq \emptyset$ then $\psi_{M_\rho}^*(\delta) \leq c_{\textit{thick}} \delta^{\gamma}$.
			\item[(iii)] 
			If $M_\rho^{-\delta}=\emptyset$, then  $|\mathcal{C}_{\tau}(A)| = 1$ for all non-empty $A\subset M_\rho^{+\delta}$ and all $\tau>2c_{\textit{thick}} \delta^\gamma$.
			\item[(iv)] 
			For each $\delta \in (0,\delta_{\textit{thick}}]$ there exists a $\rho \geq \rho_*$ with $M_{\rho}^{-\delta}=\emptyset$.
		\end{itemize}
	\end{assumption}

	\subsection{Best-scored Random Forest Density Estimation} \label{subsec::Best-scoredRandomForestDensityEstimation}

	Considering the fact that the density estimation should come first before the analysis on the level sets, we dedicate this subsection to the methodology of building an appropriate density estimator. Different from the usual histogram density estimation \citep{steinwart2015adaptive} and kernel density estimation \citep{steinwart2017adaptive}, this paper adopts a novel random forest-based density estimation strategy, namely the best-scored random forest density estimation proposed recently by \cite{hang2018best}.

	\subsubsection{Purely Random Density Tree} \label{subsubsec::PurelyRandomDensityTree}

	Recall that each tree in the best-scored random forest is established based on a purely random partition followed the idea of \cite{breiman2000some}. To give a clear description of one possible construction procedure of this purely random partition, we introduce the random vector $Q_i:=(L_i,R_i,S_i)$ as in \cite{hang2018best}, which represents the building mechanism at the $i$-th step. To be specific, 
	
		\begin{figure*}[htbp]
		\begin{minipage}[t]{0.99\textwidth}  
			\centering  
			\includegraphics[width=\textwidth]{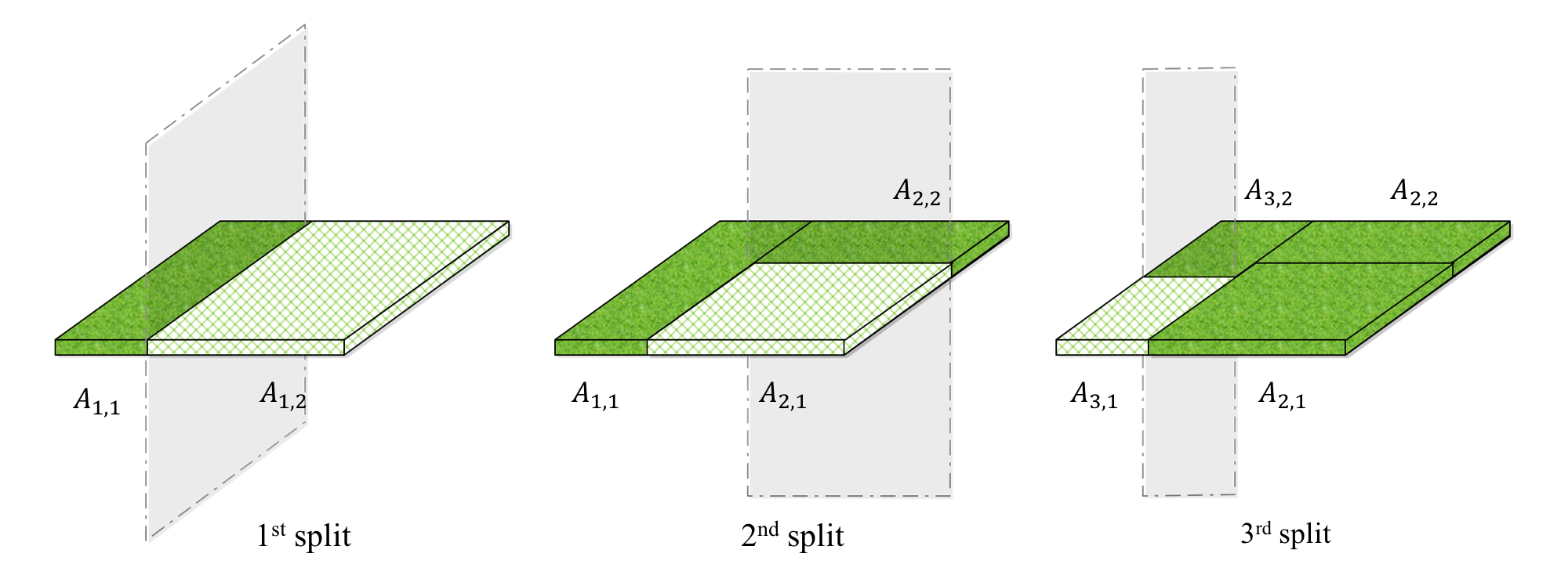}  
		\end{minipage}  
		\centering  
		\caption{Possible construction procedures of three-split axis-parallel purely random partitions in a two-dimensional space. The first split divides the input domain, e.g.~$B_r$ into two cells $A_{1, 1}$ and $A_{1, 2}$. Then, the to-be-split cell is chosen uniformly at random, say $A_{1, 2}$, and the partition becomes $A_{1,1}$, $A_{2,1}$, $A_{2,2}$ after the second random split. Finally, we once again choose one cell uniformly at random, say $A_{1, 1}$, and the third random split leads to a  partition consisting of $A_{2, 1}$, $A_{2,2}$, $A_{3, 1}$ and $A_{3, 2}$.}  
		\label{fig:ap}
	\end{figure*}

	\begin{itemize}
		\item 
	[$L_i$] denotes the to-be-split cell at the $i$-th step chosen uniformly at random from all cells formed in the $(i-1)$-th step; 
			\item 
	[$R_i$] $\in \{1, \ldots, d\}$ stands for the dimension chosen to be split from in the $i$-th step where each dimension has the same probability to be selected, that is, $\{ R_i, i \in \mathbb{N}_{+} \}$ are i.i.d.~multinomial distributed with equal probabilities; 
			\item 
	[$S_i$] is a proportional factor standing for the ratio between the length of the newly generated cell in the $R_i$-th dimension after the $i$-th split and the length of the being-cut cell $L_i$ in the $R_i$-th dimension. We emphasize that $\{S_i, i \in \mathbb{N}_{+}\}$ are i.i.d.~drawn from the uniform distribution $\mathrm{Unif}[0, 1]$.
\end{itemize}

	In this manner, the above splitting procedure leads to a so-called partition variable $Z:=(Q_1,\ldots, Q_p,\ldots) \in \mathcal{Z}$ with probability measure of $Z$ denoted by $\mathrm{P}_Z$, and any specific partition variable $Z \in \mathcal{Z}$ can be treated as a splitting criterion. Moreover, for the sake of notation clarity, we denote by $\mathcal{A}_{(Q_1, \ldots Q_p)}$ the collection of non-overlapping cells formed after conducting $p$ splits on $B_r$ following $Z$. This can be further abbreviated as $\mathcal{A}_{Z, p}$ which exactly represents a random partition on $B_r$. Accordingly, we have $\mathcal{A}_{Z, 0} := B_r$, and for certain sample $x \in B_r$, the cell where it falls is denoted by $A_{Z, p}(x)$.

	In order to better characterize the purely random density tree, we give another expression of the random partition on $B_r$, which is $\mathcal{A}_{Z, p}:=\{A_j, j=1,\ldots,p\}$ where $A_j$ represents one of the resulting cells of this partition. 
%	Here, we define the diameter of an $A \subset \mathcal{X}$ to be 
%	$\mathrm{diam}(A) := \sup_{x, x' \in A} \| x - x'\|_2$.
	Based on this partition, we can build the random density tree with respect to probability measure $Q$ on $\mathbb{R}^d$, denoted as $f_{Q, Z, p}: \mathbb {R}^d \to [0, \infty)$, defined by
	\begin{align*}
		f_{Q, Z}(x) := f_{Q, Z, p}(x):= \sum_{j=0}^p \frac{Q(A_j)\eins_{A_j}(x)}{\mu(A_j)} + \frac{Q(B_r^c)\eins_{B_r^c}(x)}{\mu(B_r^c)}
	\end{align*}
	where unless otherwise stated, we assume that for all $A \in \mathcal{A}_{Z, p}$, the Lebesgue measure $\mu(A)>0$. In this regard, when taking $Q = \mathrm{P}$, the density tree decision rule becomes
	\begin{align*}
		f_{\mathrm{P}, Z}(x) = \frac{\mathrm{P}(A(x))}{\mu(A(x))}=\frac{1}{\mu(A(x))} \int_{A(x)} f(x') d\mu(x),
		\qquad \qquad
		x \in B_r, 
	\end{align*}
	where $A(x) := A_j$. When taking $Q$ to be the empirical measure $D_n := \frac{1}{n}\sum_{i=1}^n \delta_{x_i}$, we obtain
	\begin{align*}
		D_n(A(x)) = \mathbb{E}_{D_n}\eins_{A(x)} 
		=\frac{1}{n}\sum_{i=1}^n \delta_{x_i}(A(x)) 
		=\frac{1}{n}\sum_{i=1}^n \eins_{A(x)}(x_i),
	\end{align*}
	and hence the density tree turns into
	\begin{align} \label{RandomDensityTree}
		f_{D, Z}:=f_{D_n, Z}(x) = \frac{D_n(A(x))}{\mu(A(x))} 
		= \frac{1}{n \mu(A(x))} \sum_{i=1}^n \eins_{A(x)}(x_i).
	\end{align}

	\subsubsection{Best-scored Random Density Trees and Forest} \label{subsubsec::Best-scoredRandomDensityTreesandForest}

	Considering the fact that the above partitions completely make no use of the sample information, the prediction results of their ensemble forest may not be accurate enough. 
	In order to improve the prediction accuracy, we select one partition for tree construction out of $k$ candidates with the best density estimation performance according to certain 
	performance measure such as \emph{ANLL} \citep[Section 5.4]{hang2018best}.
	The resulting trees are then called the \emph{best-scored random density trees}.
	
	Now, let $f_{D,Z_t}$, $1\leq t\leq m$ be the $m$ best-scored random density tree estimators generated by the splitting criteria $Z_1,\dots,Z_m$ respectively, which is defined by
	\begin{align*}
		f_{D,Z_t}(x) 
		:=\sum^p_{j=0} \frac{D(A_{tj})1_{A_{t_j}}(x)}{\mu(A_{t_j})}+\frac{D(B_r^c)1_{B_r^c}(x)}{\mu(B_r^c)}
	\end{align*}
	where $\mathcal{A}_{Z_t}:=\{A_{tj},\, j=0,\dots,p\}$ is a random partition of $B_r$. Then the best-scored random density forest can be formulated by
	\begin{align}\label{BRF}
		f_{D,Z_{\mathrm{E}}}(x):=\frac{1}{m} \sum^m_{t=1} f_{D,Z_t}(x),
	\end{align}
	and its population version is denoted by $f_{\mathrm{P},Z_{\mathrm{E}}}$.

	\section{A Generic Clustering Algorithm} \label{sec::AGenericClusteringAlgorithmBasedonBRDF}

	In this section, we present a generic clustering algorithm, where the clusters are
	estimated with the help of a generic level set estimator which can be specified later by histogram, kernel, or random forest density estimators.
	To this end, 
	let the optimal level $\rho^*$ and the resulting clusters $A_i^*$, $i = 1, 2$ for distributions
	be as in Definition \ref{def::ClusterProperty}, and the constant $\rho_*$ be as in Assumption \ref{ass::MainDistributionaddi}.
	The goal of this section is to investigate whether $\rho^*$ or $\rho_*$
	is possible to be estimated and $A_i^*$, $i = 1, 2$ can be clustered.
	
	Let us first recall some more notations introduced in Section \ref{sec::Preliminaries}. 
	For a $\mu$-absolutely continuous distribution $\mathrm{P}$, let
	the level $\rho^{**}$, the level set $M_{\rho}$, $\rho \geq 0$, 
	and  the function $\tau^*$
	be as in Definition \ref{def::ClusterProperty}.
	Furthermore, for a fixed set $A$, its $\delta$-tubes $A^{-\delta}$ and $A^{+\delta}$ are defined by \eqref{DeltaTubes}.
	 Moreover, 
	concerning with the thick level sets,
	the constant $\delta_{\text{thick}}$ and the function $\psi(\delta)$ are introduced 
	by Definition \ref{def::ThickLetSets}.

	In what follows, let $(L_\rho)_{\rho\geq 0}$ always be a decreasing family of sets $L_\rho\subset \mathcal{X}$ such that
	\begin{equation} \label{UncertaintyControl}
	M_{\rho+\varepsilon}^{-\delta} \subset L_{\rho} \subset M_{\rho-\varepsilon}^{+\delta}
	\end{equation}
	holds for all $\rho \in [0, \rho^{**}]$.

	The following theorem relates the component structure of a family of level sets estimators $L_\rho$, which is a decreasing family of subsets of $\mathcal{X}$, to the component structure of certain sets $M_{\rho+\varepsilon}^{-\delta}$,
	more details see e.g., \cite{steinwart2015adaptive}.
	\begin{theorem}  \label{thr::taoCondition}
		Let Assumption \ref{ass::MainDistribution} hold.
		Furthermore, 
		for $\varepsilon^*\in (0,\rho^{**}-\rho^*)$, let 
		$\varepsilon \in (0,\varepsilon^*]$,
		$\delta\in (0,\delta_{\text{thick}}]$, 
		$\tau \in (\psi(\delta),\tau^*(\varepsilon^*))$,
		and $(L_\rho)_{\rho\geq 0}$ be as in \eqref{UncertaintyControl}.
		Then, for all $\rho\in [0,\rho^{**}-3\varepsilon]$ and the corresponding CRMs $\zeta:\mathcal{C}_{\tau}(M_{\rho+\varepsilon}^{-\delta})\to \mathcal{C}_{\tau}(L_\rho)$, the following disjoint union holds:
		\begin{align*}\label{the:1}
			\mathcal{C}_\tau(L_\rho)=\zeta(\mathcal{C}_\tau(M_{\rho+\varepsilon}^{-\delta}))\cup \{B'\in \mathcal{C}_\tau(L_\rho):B'\cap L_{\rho+2\varepsilon}=\varnothing\}.
		\end{align*}
	\end{theorem}

From
	 Theorem \ref{thr::taoCondition} we see that for suitable $\varepsilon$, $\delta$, and $\tau$, all $\tau$-connected components $B'$ of $L_{\rho}$ are either contained in $\zeta(\mathcal{C}_\tau(M_{\rho+\varepsilon}^{-\delta}))$, or vanish at level $\rho+2\varepsilon$. Accordingly, carrying out these steps precisely, we obtain a generic clustering strategy shown in Algorithm \ref{alg::clustering}.

	\begin{algorithm}
		\caption{Estimate clusters with the help of a generic level set estimator}	
		\label{alg::clustering}	
		%	\SetAlgoNoLine	
		\KwIn{
			some $\tau>0$, $\varepsilon>0$ and a start level $\rho_0\geq 0$. A decreasing family $(L_{\rho})_{\rho \geq 0}$ of subsets of X.
		}
		$\rho = \rho_0$\\
		\Repeat
		{$M\neq 1$}
		{
			Identify the $\tau$-connected components $B_1',\dots,B_M'$ of $L_{\rho}$ satisfying
			\begin{align*}
				B_i' \cap L_{\rho+2\varepsilon}\neq \emptyset
			\end{align*}
			$\rho=\rho+\varepsilon$
		}
		$\rho=\rho+2\varepsilon$\\
		Identify the $\tau$-connected components $B_1',\dots,B_M'$ of $L_{\rho}$ satisfying
		\begin{align*}
			B_i'\cap L_{\rho+2\varepsilon}\neq \emptyset
		\end{align*}
		\If {$M > 1$} 
		{\Return {$\rho_{\textit{out}}=\rho$ and the sets $B_i'$ for $i=1,\ldots,M$.}}
		\Else {\Return {$\rho_{\textit{out}}=\rho_0$ and the set $L_{\rho_0}$.}}
		\KwOut{An estimator of $\rho_*$ or $\rho^*$ the corresponding clusters.}
	\end{algorithm}
	
	\begin{figure*}[htbp]
		\begin{minipage}[t]{0.99\textwidth}  
			\centering  
			\includegraphics[width=\textwidth]{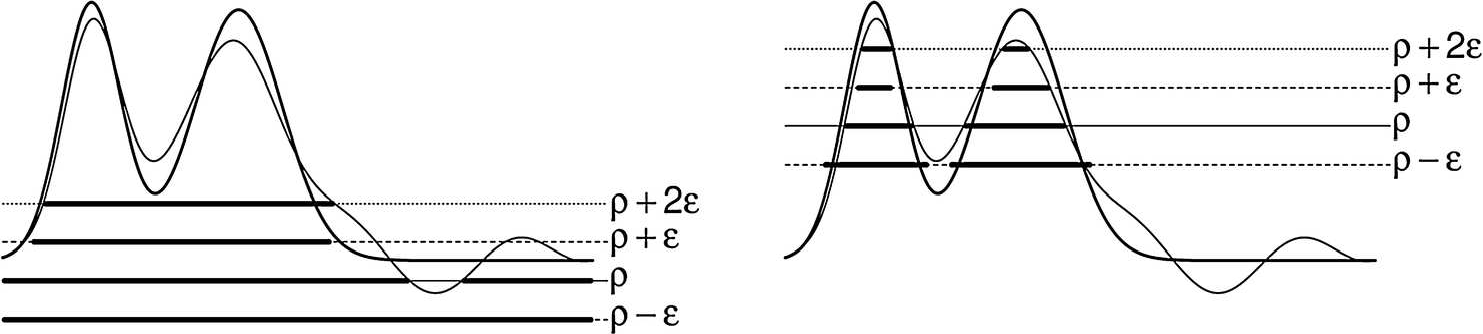}  
		\end{minipage}  
		\centering  
		\caption{Illustration of Algorithm \ref{alg::clustering}. 
		Left: The density presented by solid line has two modes on the left and a flat part on the right. A plug in approach based on a density estimator (thin solid line) with three modes is used to provide the level set estimator. The level set estimator $L_{\rho}$ satisfies \eqref{UncertaintyControl}. Only the left two components of $L_{\rho}$ do not vanish at $\rho+2\varepsilon$. Therefore, the algorithm only finds one component. 
		Right: We consider the same distribution at a higher level. In this case, both components of $L_\rho$ do not vanish at $\rho+2\varepsilon$ and thus the algorithm correctly identifies two connected components. }
		\label{fig:illualg}
	\end{figure*} 
	
	Under Assumptions \ref{ass::MainDistribution} and \ref{ass::MainDistributionaddi},
	the following theorem bounds the level $\rho_{\textit{out}}$ and the components $B_i(D)$, $i = 1, 2$, and the start level $\rho_0$ and the corresponding single cluster $M_{\rho_0} =: L_0$,
	respectively, which are outputs returned by Algorithm \ref{alg::clustering}.

\begin{theorem} \label{thr::LevelSetBound}
\textit{(i)}
Let Assumption \ref{ass::MainDistribution} hold. 
		For $\varepsilon^*\leq (\rho^{**}-\rho^*)/9$,
		let $\varepsilon \in (0,\varepsilon^*]$,
		$\delta \in (0, \delta_{\textit{thick}}]$, 
		$\tau \in (\psi(\delta),\tau^*(\varepsilon^*)]$, and
		$(L_\rho)_{\rho\geq 0}$ satisfy \eqref{UncertaintyControl} for all $\rho \geq \rho_0$.
		Then, for any data set $D$,
		 the following statements hold for Algorithm \ref{alg::clustering}:
			\begin{itemize}
				\item[(a)] The returned level $\rho_{\textit{out}}$ satisfies both $\rho_{\textit{out}}\in [\rho^*+2\varepsilon,\rho^*+\varepsilon^*+5\varepsilon]$ and
				\begin{equation*}
					\tau-\psi(\delta)<3\tau^*(\rho_{\textit{out}}-\rho^*+\varepsilon);
				\end{equation*}
				\item[(b)] The returned sets $B_i(D)$, $i = 1, 2$, can be ordered such that 
				\begin{equation} \label{equ:6}
				\sum^2_{i=1}\mu(B_i(D)\triangle A_i^*)\leq 2\sum^2_{i=1}\mu(A_i^*\setminus (A_{\rho_{\textit{out}}+\varepsilon}^i)^{-\delta})+\mu(M_{\rho_{\textit{out}}-\varepsilon}^{+\delta}\setminus \{f>\rho^*\}).
				\end{equation}
				Here, $A_{\rho_{\textit{out}}+\varepsilon}^i\in\mathcal{C}(M_{\rho_{\textit{out}}+\varepsilon})$, $i = 1, 2$, are ordered in the sense of $A_{\rho_{\textit{out}}+\varepsilon}^i\subset A_i^*$. 
				\end{itemize}
\textit{(ii)}
Let Assumption \ref{ass::MainDistributionaddi} hold. Moreover,
let $\varepsilon > 0$, $\delta \in (0, \delta_{\textit{thick}}]$ be fixed, $\tau > 2 c_{\textit{thick}} \delta^\gamma$,
and $(L_{\rho})_{\rho\geq 0}$ satisfy \eqref{UncertaintyControl} for all $\rho \geq \rho_0$. 
If $\rho_0 \geq \rho_*$,  then Algorithm \ref{alg::clustering} returns 
the start level $\rho_0$ and
the corresponding single cluster $L_0:=M_{\rho_0}$ such that
\begin{align*}
\mu(L_{\rho_0} \triangle \widehat{M}_{\rho_{*}}) \leq \mu(M_{\rho_{0}-\varepsilon}^{+\delta} \backslash \widehat{M}_{\rho_{*}})+\mu\left(\widehat{M}_{\rho_{*}} \backslash M_{\rho_{0}+\varepsilon}^{-\delta}\right)
\end{align*}
where $\widehat{M}_{\rho_{*}}:=\bigcup_{\rho>\rho_{*}} M_{\rho_{*}}$.
\end{theorem}

	The above analysis is mainly illustrated on the general cases where we assume that the underlying density has already been successfully estimated. Therefore, in the following, we delve into the characteristic of components structure and other properties of clustering algorithm under the condition where the density is estimated by the forest density estimator \eqref{BRF}.

	Note that one more notation is necessary for clear understanding: One way to define level set estimators with the help of the forest density estimator \eqref{BRF}
	is a simple plug-in approach, which is 
	\begin{align*}
		L_{D, \rho} := \{ f_{D,Z_{\mathrm{E}}}(x) \geq \rho \}.
	\end{align*}
	However, these level set estimators are too complicated to compute the $\tau$-connected components in Algorithm \ref{alg::clustering}. Instead, we take level set estimators of the form
	\begin{align}\label{LevelSetEstimator}
		L_{D,\rho} := \{ x \in D : f_{D,Z_{\mathrm{E}}}(x) \geq \rho \}^{+\sigma}.
	\end{align}

	The following theorem shows that 
	some kind of uncertainty control of the form \eqref{UncertaintyControl} is valid for
	level set estimators of the form \eqref{LevelSetEstimator}
	induced by the forest density estimator \eqref{BRF}.

	\begin{theorem} \label{theorem::UncertaintyControlForest}
		Let $\mathrm{P}$ be a $\mu$-absolutely continuous distribution on $\mathcal{X}$ and 
		$f_{D,Z_{\mathrm{E}}}(x)$ be the forest density estimator \eqref{BRF} with $\|f_{D,Z_{\mathrm{E}}}-f_{\mathrm{P},Z_{\mathrm{E}}}\|_{\infty}\leq \varepsilon$.
		For any $A \in \mathcal{A}_{Z_t,p}$, $t = 1, \ldots, m$, that is, $A$ is one of the $p+1$ cells in 
		the $t$-th partition, there exists a constant $\delta > 0$ such that
		$\mathrm{diam}(A) \leq \delta$.  
		Then, for all $\rho > 0$ and $\sigma \geq \delta$, there holds
		\begin{align} \label{UncertaintyControlForest}
			M_{\rho+\varepsilon}^{-2\sigma} 
			\subset L_{D,\rho}
			\subset M_{\rho-\varepsilon}^{+2\sigma}.
		\end{align}
	\end{theorem}

Before we present the next theorem, recall that
$r$ denotes half of the side length of the centered hypercube in $\mathbb{R}^d$ and
$m$ denotes the number of trees in the best-scored random forest.

\begin{theorem}\label{thr::BoundProbability}
		Let $\mathrm{P}$ be a $\mu$-absolutely continuous distribution on $\mathcal{X}$.
	For $r \geq 1$, $m > 0$, $\varsigma \geq 1$, $n \geq 1$,  we choose an $\varepsilon > 0$ satisfying
	\begin{align} \label{BoundVarepsilon}
	\varepsilon
	\geq \sqrt{\|f\|_{\infty} \mathcal{E}_{\varsigma, p}/ n} 
	+ \mathcal{E}_{\varsigma, p} / (3n) 
	+ 2/n,
	\end{align}
	where $\mathcal{E}_{\varsigma,p}$ is defined by
	\begin{align} \label{eq::CalE_n}
	\mathcal{E}_{\varsigma,p} := 
	128 m^2 p^{2a} \mu(B_r)^{-1} e^{2 \varsigma} 
	\bigl( (8 d + 1) (\log(4m) + \varsigma) + 23 \log n + 8 a d \log p \bigr).
	\end{align}
	Furthermore,  
	for $\delta \in (0,\delta_{\textit{thick}}/2]$ and $\tau > 0$, we choose a $\sigma$ with $\sigma\geq \delta$ and assume this $\sigma$ satisfying $\sigma <\delta_{\text{thick}}/2$ and $\psi(2\sigma)<\tau$. 
	Moreover, for each random density tree, we pick the number of splits 
	$p$ satisfying
   \begin{align} \label{BoundSplits}
   p > \bigl( 2 m d e^{\varsigma} / \delta \bigr)^{4 d / c_T}.
    \end{align}
 If we feed Algorithm \ref{alg::clustering} with parameters $\varepsilon$, $\tau$, $\sigma$, and  $(L_{D,\rho})_{\rho\geq 0}$ as in \eqref{LevelSetEstimator},
 then the following statements hold: \\
{~} \textit{(i)}
If $\mathrm{P}$ satisfies Assumption \ref{ass::MainDistribution} and 
there exists an $\varepsilon^*$ satisfying
	\begin{align*}
	\varepsilon+\inf\{\varepsilon'\in(0,\rho^{**}-\rho^*]:\tau^*(\varepsilon')\geq \tau\}
	\leq \varepsilon^*
	\leq (\rho^{**}-\rho^*) / 9,
	\end{align*}
     then with probability $\mathrm{P}^n$ not less than $1-e^{-\varsigma}$, the following statements hold:
	\begin{itemize}
		\item[(a)] 
	The returned level $\rho_{D,\textit{out}}$ satisfies both $\rho_{D,\textit{out}}\in [\rho^*+2\varepsilon,\rho^*+\varepsilon^*+5\varepsilon]$ and 
		\begin{align*}
		\tau-\psi(2\sigma)<3\tau^*(\rho_{D,\textit{out}}-\rho^*+\varepsilon);
		\end{align*}
		\item[(b)] The returned sets $B_i(D)$, $i = 1, 2$, can be ordered such that 
		\begin{align*}
		\sum^2_{i=1}\mu(B_i(D)\triangle A_i^*)\leq 2\sum^2_{i=1}\mu(A_i^*\setminus (A^i_{\rho_{D,\textit{out}}+\varepsilon})^{-2\sigma})+\mu(M^{+2\sigma}_{\rho_{D,\textit{out}}-\varepsilon}
		\setminus \{h>\rho^*\}).
		\end{align*}
		Here, $A^i_{\rho_{D,\textit{out}}+\varepsilon}\in \mathcal{C}(M_{\rho_{D,\textit{out}}+\varepsilon})$, $i = 1, 2$, are ordered in the sense of $A^i_{\rho_{D,\textit{out}}+\varepsilon}\in A_i^*$.
	\end{itemize}
{~}\textit{(ii)}
If $\mathrm{P}$ satisfies Assumption \ref{ass::MainDistributionaddi} and $\rho_0 \geq \rho^*$, then 
\begin{align*}
\mu(L_{\rho_0} \triangle \widehat{M}_{\rho_*})
\leq \mu(M^{+2\sigma}_{\rho_0-\varepsilon} \setminus \widehat{M}_{\rho_*})
+ \mu(\widehat{M}_{\rho_*} \setminus M^{-2\sigma}_{\rho_0+\varepsilon})
\end{align*}
holds with probability $\mathrm{P}^n$ not less than $1-e^{-\varsigma}$ for the returned level $\rho_0$ and the corresponding single cluster $L_0 := M_{\rho_0}$,
where $\widehat{M}_{\rho_*}:=\bigcup_{\rho>\rho_*} M_{\rho}$. 
\end{theorem}

	\section{Main Results} \label{sec::ConsistencyandRatesforRFDE-basedClustering}

	In this section, we present main theoretical results 
	of our best-scored clustering forest
	on the consistency as well as 
	convergence rates
	for both the optimal level $\rho^*$ and the true clusters $A_i^*$, $i = 1, 2$, simultaneously 
	using the error bounds derived in Theorem \ref{thr::LevelSetBound} and Theorem \ref{thr::BoundProbability}, respectively. 
	We also present some comments and discussions on the obtained theoretical results.

   \subsection{Consistency for Best-scored Clustering Forest}

   \begin{theorem}[Consistency] \label{the::Consistency}
   	Let Assumption \ref{ass::MainDistribution} hold. Furthermore, 
   	for certain constant $q \in (0, 1)$,
   	assume that $(\varepsilon_n)$, $(\tau_n)$, $(\delta_n)$, and $(\sigma_n)$ are strictly positive sequences converging to zero satisfying 
   	$\varepsilon_n\geq 2/(nq)$ for sufficiently large $n$,
   	$\sigma_n = \delta_n$, $\psi(2\sigma_n)\leq \tau_n$.
   	Moreover,
   	let the number of splits $p_n$ satisfy
   	\begin{align*}
   	\lim_{n \to \infty} n p_n^{-2a} (\log n)^{-1} \varepsilon_n^2
   	& = \infty,
   	   	\\
   	   	\lim_{n \to \infty}  \delta_n p_n^{c_T/(4d)} 
   	& = \infty,
   	\end{align*}
   	where $c_T = 0.22$ and $a = 4.33$.
   	If we feed Algorithm \ref{alg::clustering} with parameters $\varepsilon_n$, $\sigma_n$, $(L_{D,\rho})_{\rho\geq 0}$ as in \eqref{LevelSetEstimator}, and $p_n$,
   then the following statements hold:	
   	\begin{itemize}
   		   		\item[(i)]
   		If $\mathrm{P}$ satisfies Assumption \ref{ass::MainDistribution}, then for all $\epsilon > 0$,  the returned level $\rho_{D,\textit{out}}$ satisfies
   		\begin{align*}
   		\lim_{n\to \infty} \mathrm{P}^n \bigl( 
   		\{D\in X^n: 0<\rho_{D,\textit{out}}-\rho^*\leq \epsilon \} \bigr) = 1.
   		\end{align*}
   		Moreover, if $\mu(\overline{A_1^*\cup A_2^*}\setminus(A_1^*\cup A_2^*))=0$, then for all $\epsilon > 0$,
   		the returned sets $B_i(D)$, $i = 1, 2$, satisfy  
   		\begin{align*}
   		\lim_{n\to\infty} \mathrm{P}^n 
   		\biggl( \biggl\{ D \in X^n :\sum_{i=1}^2 \mu(B_i(D) \triangle A_i^*) \leq \epsilon \biggr\} \biggr) = 1.
   		\end{align*}
   		\item[(ii)] 
   		If $\mathrm{P}$ satisfies Assumption \ref{ass::MainDistributionaddi} and $\rho_*=0$, then for all $\epsilon > 0$,  the returned level $\rho_{D,\textit{out}}$ satisfies
   		\begin{align*}
   		\lim_{n\rightarrow \infty} \mathrm{P}^n(\{D\in X^n:0<\rho_{D,\textit{out}}\leq \epsilon\} )=1.
   		\end{align*}
   		Moreover, if $\mu(\overline{\{f>0\}}\setminus \{f>0\})=0$,
   		then for all $\epsilon > 0$,
   		the returned set $L_{D,\rho_{D,\textit{out}}}$ satisfies  
   		\begin{align*}
   		\lim_{n\rightarrow }\mathrm{P}^n(\{D\in X^n:\mu(L_{D,\rho_{D,\textit{out}}}\triangle \{f>0\})\leq \epsilon\})=1;
   		\end{align*}

   	\end{itemize}
   \end{theorem}

	\subsection{Convergence Rates for Best-scored Clustering Forest}

	In this subsection, we derive the convergence rates 
	for both estimation problems, that is, for estimating the optimal level $\rho^*$ and the true clusters $A_i^*$, $i = 1, 2$, in our proposed algorithm separately.

	\subsubsection{Convergence Rates for Estimating the Optimal Level}

	In order to derive the convergence rates for estimating the optimal level $\rho^*$,
	we need to make following assumption
	that describes how well the clusters are separated above 
	$\rho^*$.

	\begin{definition} \label{def::SeparationExponent}
		Let Assumption \ref{ass::MainDistribution} hold. The clusters of $\mathrm{P}$ 
		are said to have \emph{separation exponent} $\kappa \in (0,\infty]$ if there exists a constant $\underline{c}_{\textit{sep}}>0$ such that
		\begin{align*}
			\tau^*(\varepsilon)\geq \underline{c}_{\textit{sep}}\varepsilon^{1/\kappa}
		\end{align*}
		holds for all $\varepsilon\in (0,\rho^{**}-\rho^*]$. 
		Moreover, the separation exponent $\kappa$ is called \emph{exact} if there exists another constant $\overline{c}_{\textit{sep}}>0$ such that
		\begin{align*}
			\tau^*(\varepsilon)\leq \overline{c}_{\textit{sep}}\varepsilon^{1/\kappa}
		\end{align*}
		holds for all $\varepsilon\in (0,\rho^{**}-\rho^*]$.
	\end{definition}

	The \emph{separation exponent} describes how fast the connected components of the $M_\rho$ approach each other for $\rho\to \rho^*$ and a distribution having separation exponent $\kappa$ also has separation exponent $\kappa'$ for all $\kappa'<\kappa$. If the separation exponent $\kappa=\infty$, then the clusters $A_1^*$ and $A_2^*$ do not touch each other. 
	With the above Definition \ref{def::SeparationExponent}, we are able to establish error bounds for estimating the optimal level $\rho^*$ 
	in the following  theorem
	whose proof is quite similar to that of Theorem 4.3 in \cite{steinwart2015adaptive} and hence will be omitted.
	
	\begin{theorem} \label{the::rho_rates}
		Let Assumption \ref{ass::MainDistribution} hold,
       and
		assume that $\mathrm{P}$ has a bounded $\mu$-density $f$ whose clusters have separation exponent $\kappa\in (0,\infty]$. 
		For $r \geq 1$, $m > 0$, $\varsigma \geq 1$, $n \geq 1$,  we choose an $\varepsilon > 0$ satisfying
		\begin{align*}
		\varepsilon \geq 
		\sqrt{\|f\|_{\infty} \mathcal{E}_{\varsigma, p}/ n} 
		+ \mathcal{E}_{\varsigma, p} / (3n) 
		+ 2/n,
		\end{align*}
		with $\mathcal{E}_{\varsigma,p}$ as in \eqref{eq::CalE_n}.
		Furthermore,  
		for $\delta \in (0,\delta_{\textit{thick}}/2]$ and $\tau > 0$, we choose a $\sigma$ with $\sigma\geq \delta$ and assume this $\sigma$ satisfying $\sigma <\delta_{\text{thick}}/2$ and $\psi(2\sigma)<\tau/2$. 
		Moreover, for each random density tree,
		we pick
		the number of splits $p$ satisfying
		\begin{align*}
			p > \bigl( 2 m d (K+2r)e^\varsigma / \delta \bigr)^{4 d / c_T}.
		\end{align*}
		Finally, suppose that
		$\varepsilon^*:=\varepsilon+(\tau/\underline{c}_{\textit{sep}})^{\kappa}$ satisfies $\varepsilon^*\leq (\rho^{**}-\rho^*)/9$.
		If we feed
		Algorithm \ref{alg::clustering} with 
		parameters $\varepsilon$, $\tau$, $\sigma$, $(L_{D,\rho})_{\rho\geq 0}$ 
		as in \eqref{LevelSetEstimator}, and $p_n$, 
		then  the returned level $\rho_{D,\textit{out}}$ satisfies
		\begin{align} \label{RhoExessError}
			\varepsilon
			< \rho_{D,\textit{out}}-\rho^*
            \leq (\tau/\underline{c}_{\textit{sep}})^{\kappa}+6\varepsilon
		\end{align}
		with probability $\mathrm{P}^n$ 
		not less than $1-e^{-\varsigma}$. 
		Moreover, if the separation exponent $\kappa$ is exact and $\kappa<\infty$, then we have
		\begin{align}\label{equ::tau_rho}
			\rho_{D,\textit{out}}-\rho^*
			> \bigl( \tau / (6 \overline{c}_{\textit{sep}} )  \bigr)^{\kappa} / 4.
		\end{align}
	\end{theorem}

	\begin{corollary}[Convergence Rates for Estimating the Optimal Level] \label{col::rho_rates}
		Let Assumption \ref{ass::MainDistribution} hold and suppose that $f$ is $\alpha$-H\"{o}lder continuous with exponent $\alpha \in (0,1]$ whose clusters have separation exponent $\kappa\in (0,\infty)$. 
		For any $\epsilon > 0$, and all $n \geq 1$, let $(\varepsilon_n)$, $(\tau_n)$, $(\delta_n)$, and $(\sigma_n)$ be sequences with
		\begin{align*}
		\varepsilon_n & = \bigl( n^{- \lambda \alpha} (\log n)^{2+\lambda \alpha} \log \log n \bigr)^{\gamma\kappa/(2(\gamma\kappa+\epsilon))},
				\\
		\tau_n & = \bigl( n^{- \lambda \alpha} (\log n)^{2+\lambda \alpha}\log\log n \bigr)^{\gamma/2(\gamma\kappa+\epsilon)},
		\\
		\sigma_n & = \delta_n = \bigl( n^{- \lambda} (\log n)^{2+\lambda} \log \log \log n \bigr)^{1/(2(\gamma\kappa+\epsilon))},
		\end{align*}
		where $\lambda=c_T/(c_T\alpha+4ad)$, $c_T = 0.22$ and $a = 4.33$. 
		Moreover, we choose the number of splits as
		\begin{align*}
			p_n = (n/\log n)^{2d/(c_T \alpha + 4 a d)}.
		\end{align*}
		If we feed Algorithm \ref{alg::clustering} with parameters $\varepsilon_n$, $\tau_n$, $\sigma_n$, $(L_{D,\rho})_{\rho\geq 0}$, and $p_n$, 
		then for all sufficiently large n,	
		there exists a constant $\overline{c} \geq 1$ such that the returned level $\rho_{D,\textit{out}}$
		 satisfies
		\begin{align*}
			\mathrm{P}^n( \rho_{D,\textit{out}}-\rho^*\leq \overline{c}\varepsilon_n )\geq 1 - 1/\log n.
		\end{align*}
		Moreover, if the separation exponent $\kappa$ is exact, there exists another constant $\underline{c} \geq 1$ such that for all sufficiently large $n$, there holds
		\begin{align*}
			\mathrm{P}^n( \underline{c}\varepsilon_n\leq \rho_{D,\textit{out}}-\rho^*\leq \overline{c}\varepsilon_n) \geq 1 - 1/\log n.
		\end{align*}
	\end{corollary}

	\subsubsection{Convergence Rates for Estimating the True Clusters}

	Our next goal is to
	establish learning rates for the true clusters, in other words, describing how fast 
	$\sum_{i=1}^2 \mu( B_i(D) \triangle A_i^*)$ goes to $0$. 
	On account that this is a modified level set estimation problem, we need to make some further assumptions on $\mathrm{P}$. The first definition can be considered as a one-sided variant of a well-known condition introduced by \cite[Theorem 3.6]{polonik1995measuring}. 
	
	\begin{definition}
		Let $\mu$ be a measure on $\mathcal{X}$ and $\mathrm{P}$ be a distribution on $\mathcal{X}$ that has a $\mu$-density $f$. For a given level $\rho\geq 0$, we say that \emph{$\mathrm{P}$ has flatness exponent $\vartheta \in (0,\infty]$} if there exists a constant $c_{\textit{flat}}>0$ such that for all $s > 0$, we have
		\begin{align} \label{ineq::flat}
			\mu(\{0<f-\rho<s\})\leq (c_{\textit{flat}}s)^\vartheta.
		\end{align}
	\end{definition}
	It can be easily observed from \eqref{ineq::flat} that the larger $\vartheta$ is, the steeper $f$ approaches $\rho$ from above. Particularly, in the case of $\vartheta = \infty$, the density $f$ is allowed to take the value $\rho$, otherwise it would be bounded away from $\rho$.

	The next definition describes the roughness of the boundary of the clusters, see also Definition 4.6 in \cite{steinwart2015adaptive}.
	
	\begin{definition} \label{def::SmoothBoundary}
		Let Assumption \ref{ass::MainDistribution} hold. Given some $\alpha_0 \in(0,1]$, we say that \emph{the clusters have an $\alpha_0$-smooth boundary} if there exists a constant $c_{\textit{bound}}>0$ such that 
		for all $\rho \in (\rho^*,\rho^{**}]$ and
		$\delta \in (0,\delta_{\textit{thick}}]$, there holds
		\begin{align*}
			\mu((A_\rho^i)^{+\delta}\setminus (A_\rho^i)^{-\delta})\leq c_{\textit{bound}}\delta^{\alpha_0},
			\qquad
			i = 1, 2,
		\end{align*}
		where $A_\rho^i$, $i = 1, 2$ denote the connected components of the level set $M_\rho$. 
	\end{definition}

	Note that considering $\alpha > 1$ does not make sense in $\mathbb{R}^d$ and if $A \subset \mathbb{R}^d$ has rectifiable boundary, we always have $\alpha = 1$, see Lemma A.10.4 in \cite{steinwart2015suppA}.
	
	Now, we summarize all the conditions on $\mathrm{P}$ needed to obtain learning rates for cluster estimation.

	\begin{assumption} \label{ass::SetsAssumption}
		Let Assumption \ref{ass::MainDistribution} hold.
		Moreover, assume that 
		$\mathrm{P}$ has a bounded $\mu$-density $f$
		and
		a flatness exponent $\vartheta \in(0,\infty]$ at level $\rho^*$,
		whose clusters have 
		an $\alpha_0$-smooth boundary for some $\alpha_0\in(0,1]$ and 
		a separation exponent $\kappa \in (0,\infty]$.
	\end{assumption}
	
	\begin{figure*}[htbp]
		\begin{minipage}[t]{0.99\textwidth}  
			\centering  
			\includegraphics[width=\textwidth]{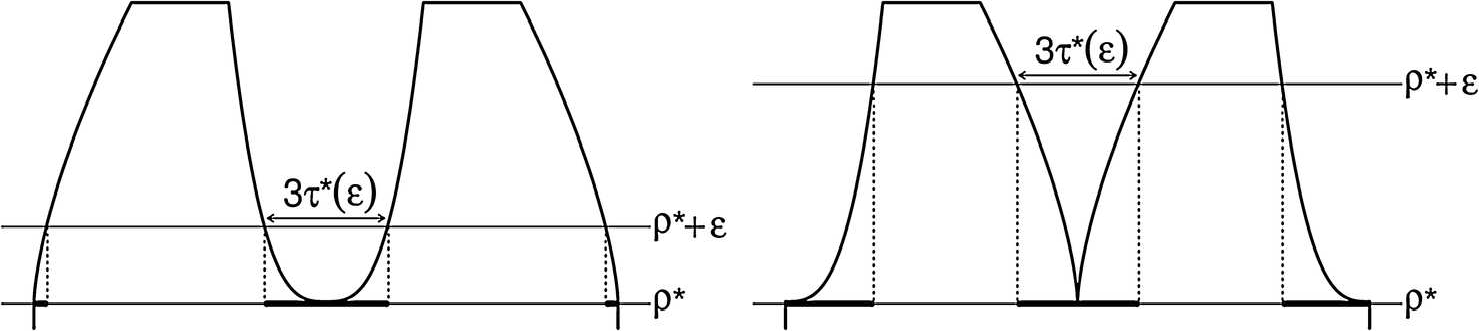}  
		\end{minipage}  
		\centering  
		\caption{Separation and flatness. 
		These two figures illustrate two possible shapes of the density $f$. The bold horizontal line indicates the set $\{\rho^*<f<\rho^*+\varepsilon\}$ and $3\tau^*(\varepsilon)$ describes the width of the valley at level $\rho^*+\varepsilon$. The value of $\varepsilon$ is chosen such that $3\tau^*(\varepsilon)$ on the right equals the value of the left. The density on the right has a narrower valley than that on the left. Therefore, $\varepsilon$ needs to be chosen larger. Moreover, it becomes more difficult to estimate the optimal level $\rho^*$ and the two clusters.}
		\label{fig:sepandflat}
	\end{figure*} 
	
	The following theorem provides a finite sample bound
	that can be later used to describe
	how well our algorithm estimates the true clusters $A_i^*$, $i = 1, 2$,
	see also Theorem 4.7
	in \cite{steinwart2015adaptive}.
	
	\begin{theorem} \label{the::sets_rate}
		Let Assumption \ref{ass::SetsAssumption} hold. 
		Furthermore, let $\varepsilon^*$  be defined as in Theorem \ref{the::rho_rates}.
		For $r > 1$, $m > 0$, $\varsigma \geq 1$, $n \geq 1$, 
		if we feed Algorithm \ref{alg::clustering} with parameters $\varepsilon$, $\tau$, $\sigma$, $(L_{D,\rho})_{\rho\geq 0}$, and $p$
		as in Theorem \ref{the::rho_rates}, then 
		the returned level $\rho_{D,\textit{out}}$
		satisfy
		inequalities \eqref{RhoExessError} and
		the returned sets
		$B_i(D)$, $i = 1, 2$ satisfy
		\begin{align*}
			\sum_{i=1}^2 \mu( B_i(D) \triangle A_i^*) 
			\leq \bigl( 7 c_{\textit{flat}} \varepsilon 
			                + c_{\textit{flat}} ( \tau/\underline{c}_{\textit{sep}})^{\kappa} \bigr)^{\vartheta} 
			         + 6 c_{\textit{bound}} (2 \sigma)^{\alpha_0}
		\end{align*}
		with probability $\mathrm{P}^n$ not less than $1-e^{-\varsigma}$.  
	\end{theorem}

		Note that if the separation exponent $\kappa$ is exact and finite, then the inequality \eqref{equ::tau_rho} also holds for the returned level $\rho_{D,\textit{out}}$.
	Moreover, if $\vartheta$ and $\kappa$ are of finite values, then the bound in Theorem \ref{the::sets_rate} behaves like 
	$$
	\varepsilon^{\vartheta} + \tau^{\vartheta \kappa} + \delta^{\alpha_0}
	$$ 
	and the convergence rates are presented in the following corollary.

	\begin{corollary}[Convergence Rates for Estimating the True Clusters] \label{col::sets_rate}
		Let Assumption \ref{ass::SetsAssumption} hold.
        Furthermore, for $n \geq 1$, 
        let 
        $(\varepsilon_n)$, $(\tau_n)$, $(\delta_n)$, $(\sigma_n)$, and $(p_n)$    
        be sequences with
		\begin{align*}
			\varepsilon_n & = \bigl( n^{- \lambda \alpha} (\log n)^{2+\lambda \alpha} \log \log n \bigr)^{\varrho/(2(\varrho+\vartheta))},
			\\
						\tau_n & = \bigl( n^{- \lambda \alpha} (\log n)^{2+\lambda \alpha} \log \log n \bigr)^{\vartheta\gamma/(2(\varrho+\vartheta))},
			\\
			\sigma_n & = \delta_n = \bigl( n^{- \lambda} (\log n)^{2+\lambda} \log \log \log n  \bigr)^{\vartheta/(2(\varrho+\vartheta))},
			\\
						p_n & = (n / \log n)^{2d/(c_T\alpha+4ad)},
		\end{align*}
		where $\lambda=c_T/(c_T\alpha+4ad)$, 
		$c_T=0.22$, $a=4.33$, 
		and  $\varrho:=\min\{\alpha_0,\vartheta\gamma\kappa\}$. 
	If we feed Algorithm \ref{alg::clustering} with parameters $\varepsilon_n$, $\tau_n$, $\sigma_n$, $(L_{D,\rho})_{\rho\geq 0}$ as in \eqref{LevelSetEstimator}, and $p_n$, then 
	there exists a constant $c \geq 1$ such that
the returned sets
$B_i(D)$, $i = 1, 2$, satisfy
		\begin{align*}
			\mathrm{P}^n \biggl( D : \sum_{i=1}^2 \mu(B_i(D) \triangle A_i^*) 
			\leq c \bigl( n^{- \lambda \alpha} (\log n)^{2+\lambda \alpha}\log\log n \bigr)^{\frac{\varrho\vartheta}{2(\varrho+\vartheta)}} \biggr)
			\geq 1 - 1/\log n.
		\end{align*}
	\end{corollary}

	\subsection{Comments and Discussions}

	This subsection presents some comments and discussions on the established learning rates for estimating the optimal level $\rho^*$ and the true clusters 
	$A_i^*$, $i = 1, 2$.
	
	First of all, let us compare our convergence rates for estimating the optimal level with existing convergence rates in the literature. 
	Corollary \ref{col::rho_rates} tells us that
	for any $\epsilon > 0$,
	our learning rate is of the form
	\begin{align*}
	n^{- \frac{\gamma \kappa}{2(\gamma\kappa+\epsilon)}
	     \cdot \frac{c_T \alpha}{c_T\alpha+4ad}},
	\end{align*}
	where $c_T=0.22$, $a=4.33$. 
	In contrast, 
	\cite{steinwart2015adaptive} has shown that
	the clustering algorithm using histogram density estimator learns with the rate
	\begin{align*}
	n^{-\frac{\gamma\kappa}{2\gamma\kappa+d}}.
	\end{align*}
	Simple algebraic calculations show that 	
	if $\epsilon$ is sufficiently small and $8 a \gamma \kappa < c_T \alpha$, then 
	this rate will be slower than ours.	
	However, if the best separation exponent $\kappa = \infty$, that is, the clusters $A_1^*$ and $A_2^*$ do not touch each other, then our learning rate becomes 
	\begin{align*}
	n^{-\frac{c_T \alpha}{2(c_T\alpha+4ad)}}
	\end{align*}
	which turns out to be slower than the rate $n^{-\frac{1}{2}}$ established in \cite{steinwart2015adaptive}.
	
    On the other hand, concerning with the learning rates for 
    estimating the true clusters, 
    Corollary \ref{col::sets_rate} shows that our algorithm learns with rate
	\begin{align*}
	n^{-\frac{\varrho\vartheta}{2(\varrho+\vartheta)} \cdot 
	       \frac{c_T\alpha}{c_T\alpha+4ad}},
	\end{align*}
	where $c_T=0.22$ and $a=4.33$. 
	Obviously, this rate is strictly slower than the rate 
	$n^{-\frac{\varrho\vartheta}{2\varrho+\vartheta d}}$
	derived by \cite{steinwart2015adaptive}. Nevertheless, 
	in the case of
	$d \varrho \vartheta c_T\geq 2(4ad\varrho+c_T\alpha\vartheta+4\alpha d\vartheta)$,
	it can be easily shown that our rate is faster than 
	the rate $n^{-\frac{\alpha}{2\alpha+d}}$ established in \cite{sriperumbudur2012consistency}.

    Note that if Assumption \ref{ass::SetsAssumption} holds with
$\alpha_0 = 1$ and $\varrho \gamma \kappa \leq 1$, then the convergence rates for estimating $\rho^*$ and the clusters can be achieved simultaneously.
In contrast, in the case of $\varrho \gamma\kappa > 1$, 
the estimation of $\rho^*$ is easier than the estimation of the level set $M_{\rho^*}$, more detailed discussion can be found in \cite{steinwart2015adaptive}.

Finally, we mention that in general, our convergence rates can be slower than other clustering algorithms due to the nature of random partition, which in turn leads to diversity and thus accuracy of our clustering algorithm.

	\section{Experimental Performance} \label{sec::ExperimentalPerformance}

	In this section, we first summarize the proposed best-scored clustering forest algorithm in Subsection \ref{subsec::AlgorithmConstruction}, and
	discuss the model selection problem of various clustering algorithms  in Subsection \ref{subsec::ExperimentalSetup}. 
	Then we compare our clustering algorithm with other proposals 
	both on synthetic data in Subsection \ref{subsec::SyntheticData}
	and real data sets in Subsection \ref{subsec::RealDataAnalysis}, respectively.

	\subsection{Algorithm Construction} \label{subsec::AlgorithmConstruction}

	\begin{algorithm}[h]
		\caption{Estimate clusters by best-scored random forest density estimation}	
		\label{alg::clustering1}	
		\SetAlgoNoLine	
		\KwIn{
			$D=\{x_1,...x_n\}$, number of density trees $m$, some ratio $r>0$, $q>0$, some positive integer $k$, $k_N$, and $k_c$.
		}
		$\hat{f}(\cdot)$ $\leftarrow$ density estimate by random forest with the number of splits $\lfloor n*r \rfloor$ and the number $m$ of best-scored density trees each generated from $k$ random trees based on $\{x_1,\ldots,x_n\}$.\\
		$\hat{D}\leftarrow \{x_j:\hat{f}(x_j)>\hat{f}(x)_q\}$, where $\hat{f}(x)_q$ denotes the $q$-quantile of $\{\hat{f}(x_i),i=1,\ldots,n\}$.\\
		$G$ $\leftarrow$ $\varepsilon$ similarity graph on $\hat{D}$.\\
		$j$ $\leftarrow$ 0\\
		\Repeat
		{$M=k_c$}
		{
			$j \leftarrow j+1$;\\
			$\lambda_j\leftarrow\hat{f}(x)_{(j)}$, where $\hat{f}(x)_{(j)}$ denotes the $j$-th smallest value of $\{\hat{f}(x):x\in\hat{D}\}$;\\
			$L_{\lambda_j}\leftarrow\{x_i\in \hat{D}:\hat{f}(x_i)\geq \lambda_j\}$;\\
			$G_j \leftarrow $ subgraph of $G$ induced by $L_j$;\\
			Identify the connected components $B_1',\dots,B_M'$ of $G_j$.
		}
		Allocate background points to these clusters with $k_N$-nearest neighbor classification.\\
		\KwOut{$\rho_{D}^*:=\lambda_j$ and $k_c$ clusters.}
	\end{algorithm}

	Our proposed best-scored clustering forest algorithm 
	is presented in detail in Algorithm \ref{alg::clustering1}.	
	In order to measure the similarity between two data clusterings, we adopt the adjusted rand index (ARI) through all experiments which can be formulated as follows:
	Given a set $S$ of n elements and two clusterings of these elements, namely $X=\{X_1,X_2,\ldots,X_r\}$ and $Y=\{Y_1,Y_2,\ldots,Y_s\}$, the overlap between $X$ and $Y$ can be summarized with $n_{ij}$ which stands for the number of objects in set $X_i\cap Y_j$, 
	$a_i=\sum^s_{j=1} n_{ij}$, and $b_i=\sum^r_{j=1} n_{ji}$.
	Then the Adjusted Rand Index is defined as 
	\begin{align*}
		\mathrm{ARI}
		= \frac{\displaystyle \sum_{ij} \binom {n_{ij}} 2 - \biggl[ \sum_i \binom {a_i} 2 \sum_j \binom {b_j} 2 \biggr] \bigg/ \binom n 2}{\displaystyle \frac{1}{2} \biggl[ \sum_i \binom {a_i} 2+ \sum_j \binom {b_j} 2 \biggr] - \biggl[ \sum_i \binom {a_i} 2 \sum_j \binom {b_j} 2 \biggr] \bigg/ \binom n 2}.
	\end{align*}

	\begin{figure*}[htbp] 
		\centering
		\begin{minipage}[b]{0.45\textwidth}
			\centering
			\includegraphics[width=\textwidth]{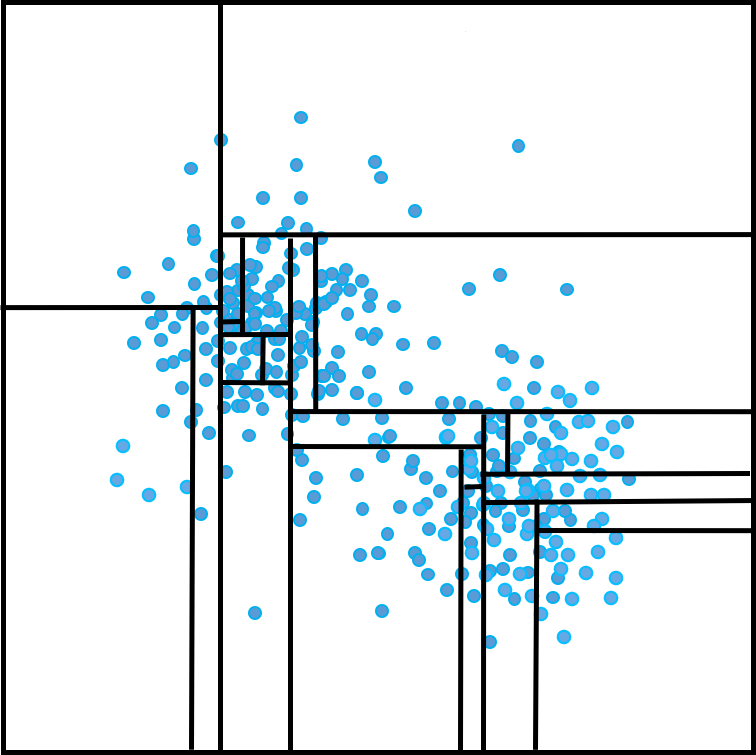}
			\centering
		\end{minipage}
		\qquad
		\begin{minipage}[b]{0.45\textwidth}
			\centering
			\includegraphics[width=\textwidth]{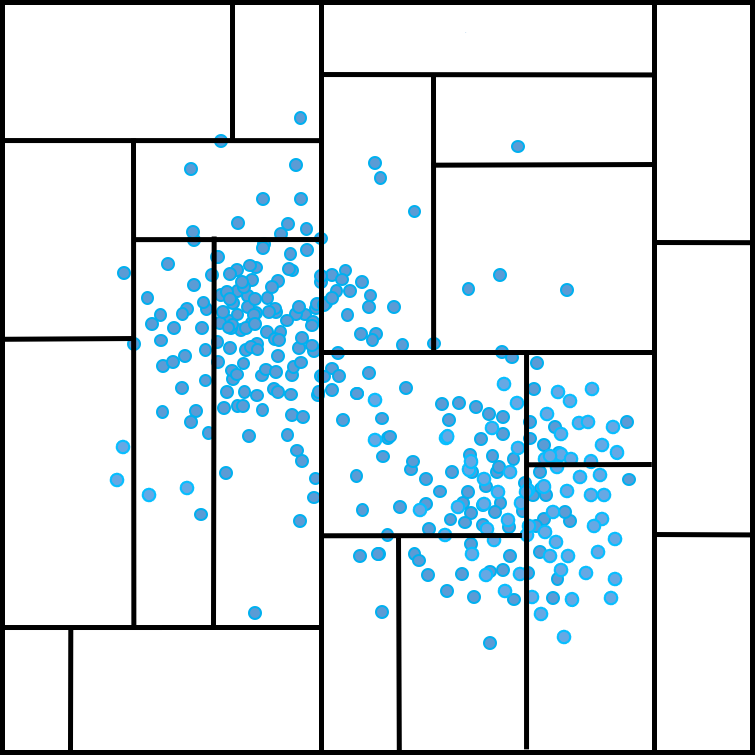}
		\end{minipage}
		\caption{Left: Adaptive method \emph{vs.} Right: Purely random method}
		\label{fig:adaptiveVSpurelyrandom}
	\end{figure*}

	Finally, it's worth mentioning that 
	in order to improve the efficiency and accuracy of our clustering algorithm,
	we also employ the adaptive splitting method (see Figure \ref{fig:adaptiveVSpurelyrandom})
	which is proposed for the density estimation problem, 
	more details please refer to Section 5.1 in \cite{hang2018best}.

	\subsection{Experimental Setup} \label{subsec::ExperimentalSetup}

	In our experiments, we compare the clusters with true classes generated by computing the following performance measures ARI (adjusted rand index) of different approaches. We conduct comparisons among some baseline density-based methods including Fast Clustering Using Adaptive Density Peak Detection (ADP-Cluster), Density-Based Spatial Clustering of Applications with Noise (DBSCAN), $k$-means and PDF-Cluster.

	\begin{itemize}
		\item 
		ADP-Cluster: The algorithm is built and improved upon the idea of \cite{2016Fast} by finding density peaks in a density-distance plot generated from local multivariate Gaussian density estimation. There are two leading parameters: the bandwidths of the multivariate kernel density and the number of the clusters $k$ determined automatically by validation criterion.
		\item 
		DBSCAN: The algorithm can be traced back to \cite{Ester1996A}. It is also a density-based clustering non-parametric algorithm while it groups points that are closely packed together. The algorithm requires two parameters: $\varepsilon$ and the minimum number of points {\tt minPts} required to form a dense region.
		\item 
		$k$-means: The only parameter in $k$-means is the number of cluster $k$. The idea goes back to \cite{Macqueen67somemethods}
		and is popular for cluster analysis in data mining. It is significant to run diagnostic checks for determining the number of clusters in the data set.
		\item 
		PDF-Cluster: The leading parameters in the algorithm are $h$ as bandwidth of kernel density estimation selected by least-square cross validation and $\lambda$ as tolerance threshold to set edges between two observations. The idea was proposed by \cite{Menardi2014An} developing a viable solution to the problem of finding connected sets in higher dimensional spaces.
	\end{itemize}

	To notify, more free parameters are alternative in the best-score clustering forest algorithm compared with other methods. To be specific, these free parameters include the number of density trees in the forest $m$, the ratio of number of splits for trees in the forest to the sample size $r$, the positive number $q$ for selecting low-density points as background points, the positive integer $k_N$ to allocate background points to clusters with $k$-NN classification as well as the number of clusters $k_c$.

	For DBSCAN, the parameter $\varepsilon$ is picked from $0.01$ to $0.30$ by $0.01$, {\tt minPts} is default and $k$ is picked from $\{1,2,3,4,5\}$. For $k$-means, the parameter $k$ is selected from $\{2,3,4,5,6,7,8,9,10\}$. For PDF-Cluster, the parameter $\lambda$ is selected from 0.01 to 0.51 by 0.01 and for our method, the parameter $m$ is set to be 100, the ratio $r$ is selected from $\{0.05, 0.1, 0.15, 0.2, 0.3, 0.4, 0.5, 0.6, 0.7, 0.8\}$ and $\varepsilon$ is selected from $q_{\varepsilon}$-quantile of the pairwise distances $\{\|x_i-x_j\|_2,1\leq i<j \leq n\}$, where $q_\varepsilon$ is chosen from $\{ 0.01$, $0.03$, $0.05$, $0.07$, $0.09$, $0.12$, $0.15$, $0.20 \}$, the parameter $k_N$ of $k$-NN is selected from $\{1,2,5\}$ and the number of clusters $k_c$ is selected from $\{2,3,4,5,6\}$. It's worth pointing out that both DBSCAN and our method assigns only a fraction of points to clusters (the foreground points), while leaving low-density observations (background points) unlabeled. Therefore, assigning the background points to clusters can be done with $k$-NN algorithm.
	In our experiment, for the algorithm with determined results, the performance is reported with the best parameter setting while for the algorithm with stochastic results, the experiment is repeated 10 times and the average performance is reported with the best parameter setting.

	We simply use the Python-package scikit-learn for DBSCAN and $k$-means and R package for ADP-Cluster and PDF-Cluster.

	\subsection{Synthetic Data}  \label{subsec::SyntheticData}

	In this subsection, we apply the density-based clustering methods mentioned above
	on four artificial examples. To be specific, we simulate four two-dimensional toy datasets with different shapes of clusters: 
	\begin{itemize}
		\item
	\textbf{noisy circles}: contains a large circle containing a smaller circle with 
		two-dimensional noise;
		\item
	\textbf{varied blob}: is generated by isotropic Gaussian blobs with variant variances for clustering;
		\item
	\textbf{noisy moons}: is made up of two interleaving half circles adding standard deviation of Gaussian noise;
		\item
	\textbf{aniso-bolb}: is anisotropicly distributed, i.e., the data set is generated by anisotropic Gaussian blobs.
	\end{itemize}
	In order to see the scalability of these algorithms, we choose the size big enough ($n = 1500$), but not too big to avoid too long running time.

	   \begin{figure*}[htbp]
		\begin{minipage}[t]{0.49\textwidth}  
			\centering  
			\includegraphics[width=\textwidth]{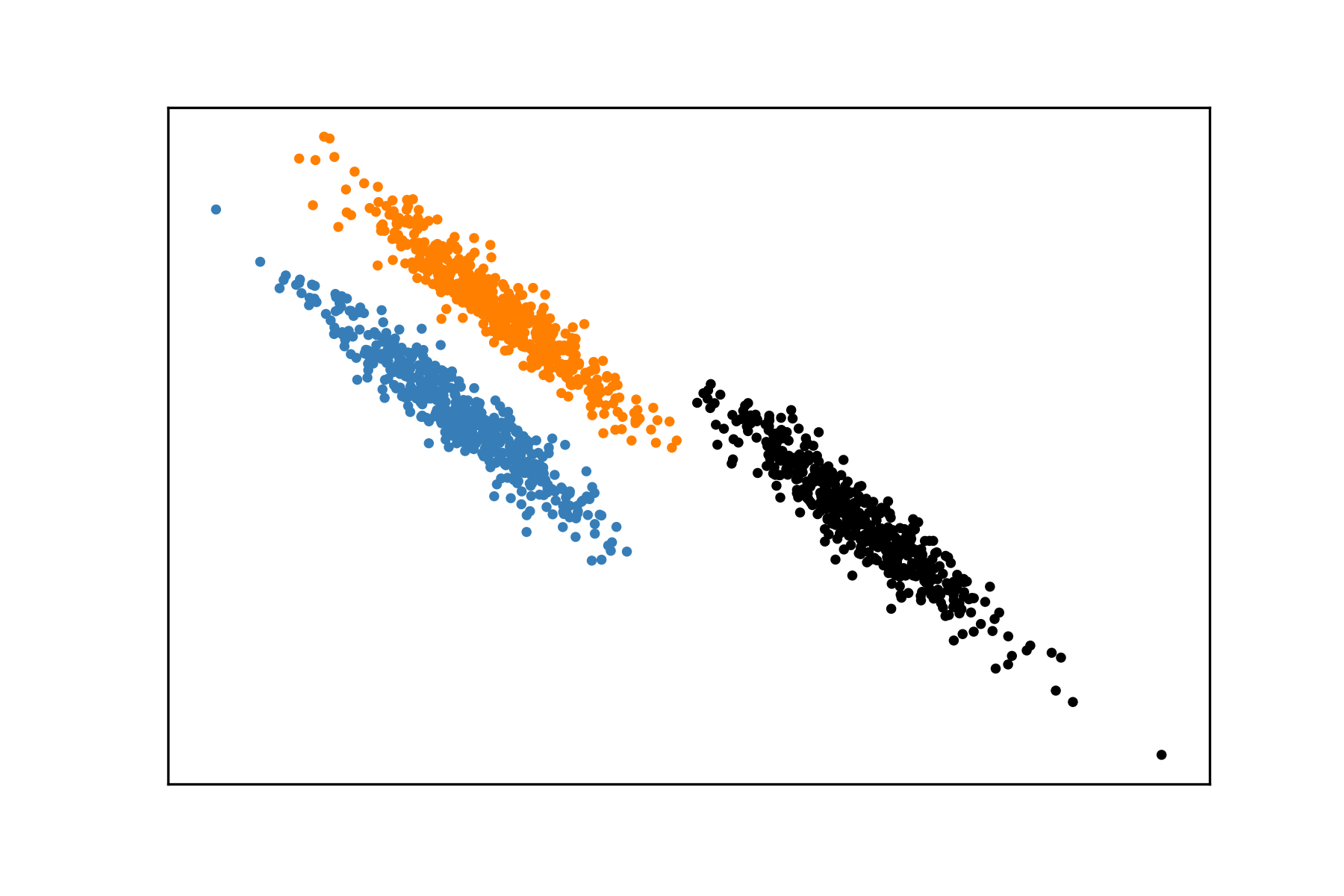} \includegraphics[width=\textwidth]{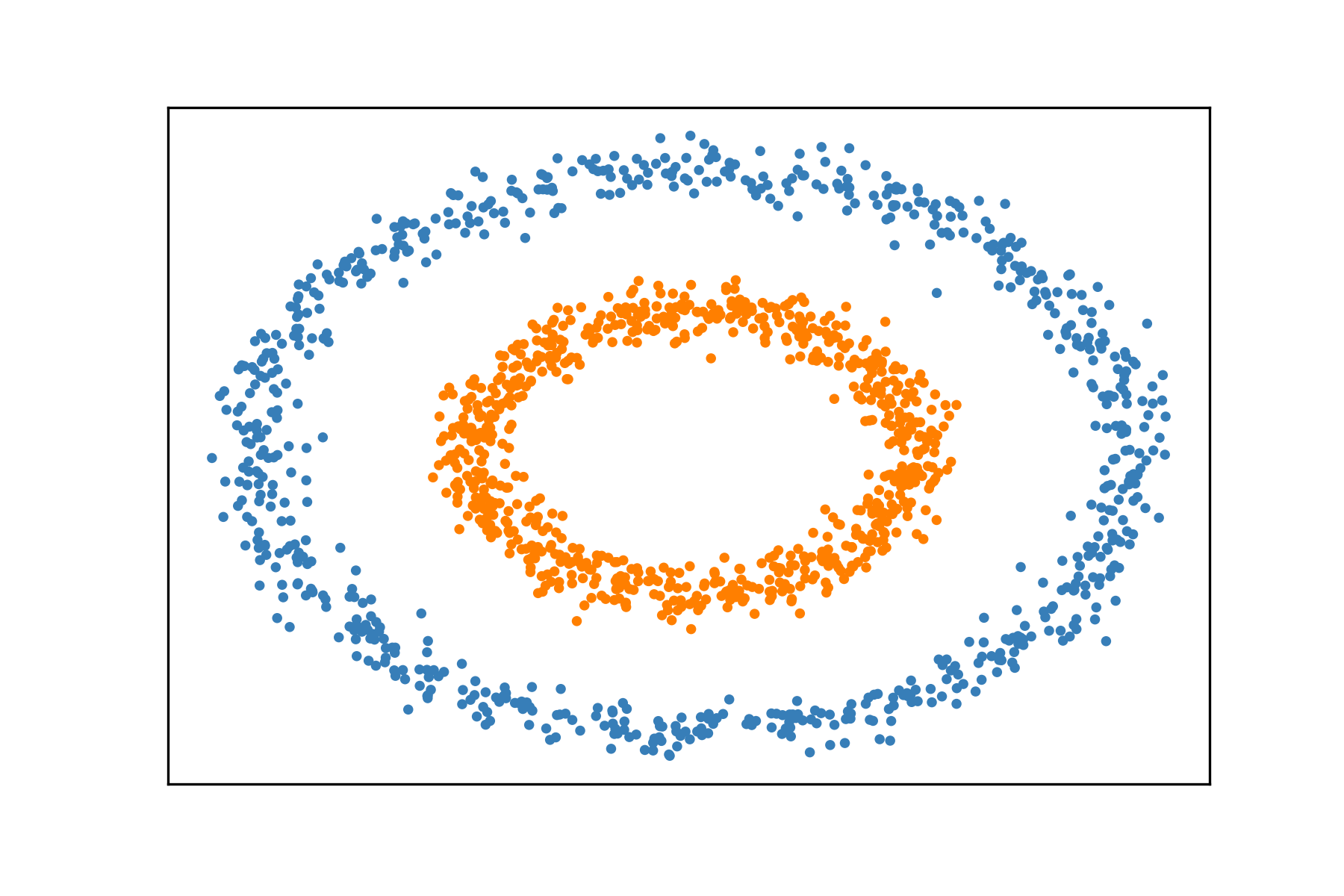}	
		\end{minipage} 
		\begin{minipage}[t]{0.49\textwidth}  
			\includegraphics[width=\textwidth]{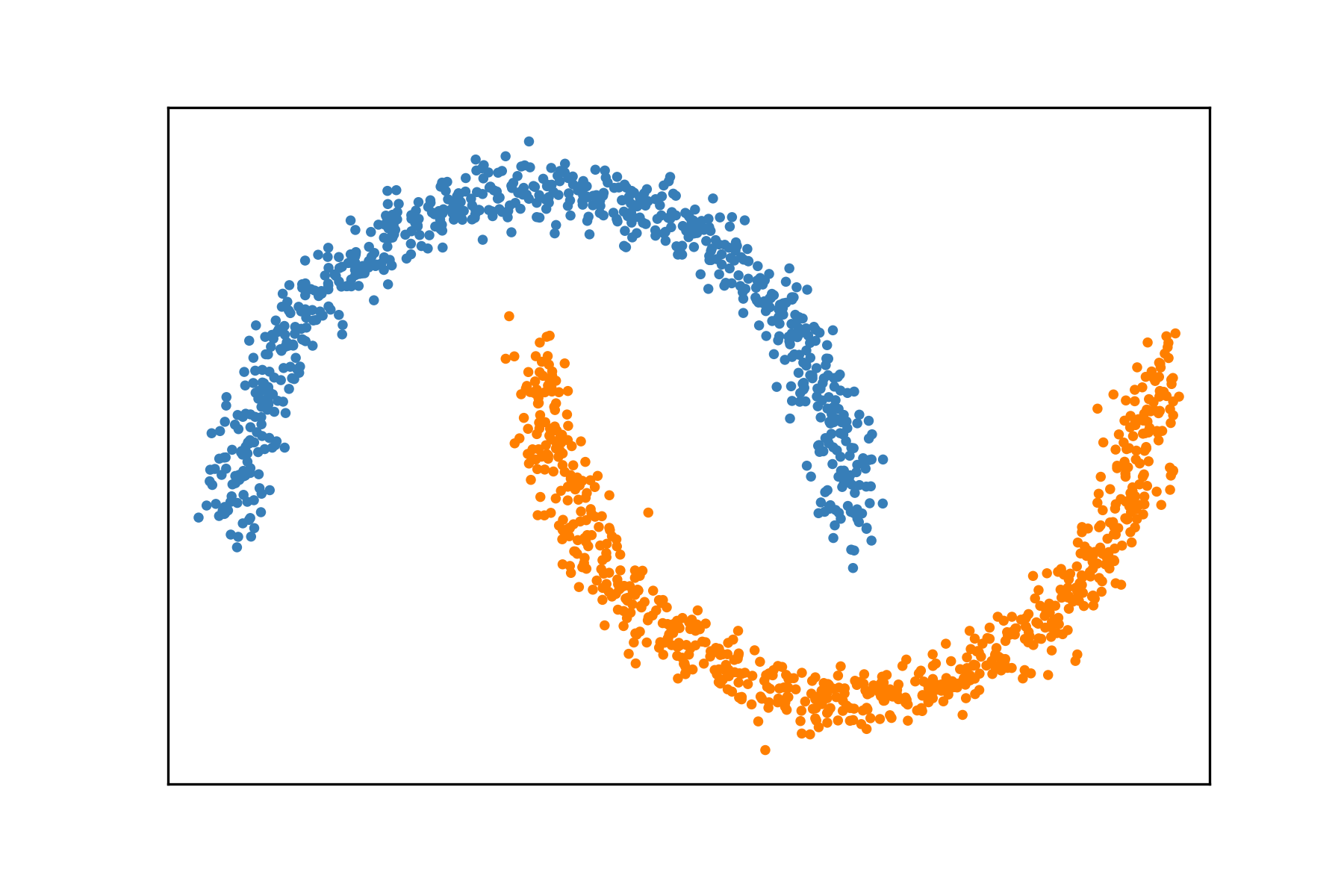} \includegraphics[width=\textwidth]{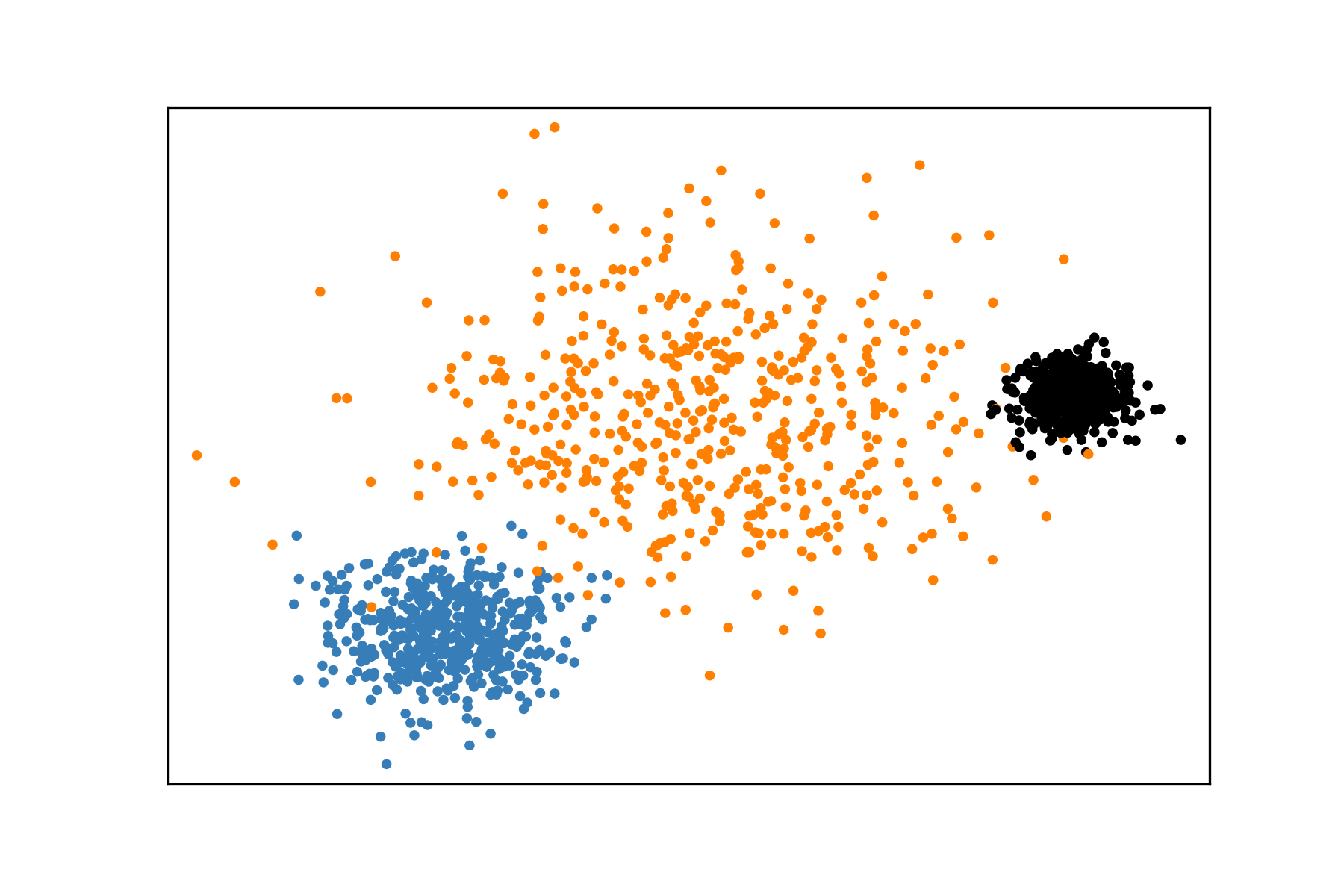}			  
		\end{minipage}  
		\centering  
		\caption{Synthetic data. The picture at the upper left shows \textit{aniso-blob} data; the picture at the upper right shows \textit{noise moons} data; the picture at the bottom left shows \textit{noise circles} data and the picture at the bottom right shows \textit{varied} data. The clusters are distinguished by different colors.}
		\label{fig:synthetic data}
	\end{figure*} 
	
	\begin{table*}[h] 
		\setlength{\tabcolsep}{9pt}
		\centering
		\captionsetup{justification=centering}
		\caption{\footnotesize{Average ARI (adjusted rand index) over Four simulated data sets}}
		\resizebox{\textwidth}{13mm}{
			\begin{tabular}{c|ccccc}
				\toprule
				Datasets & ADP-Cluster	& DBSCAN &	$k$-means & PDF-Cluster	& Ours
				\\
				\hline 
				\hline
				\text{aniso-blob}	& 0.519693602	&\textbf{1}	&0.590986466	&0.992013365 &	\textbf{1}
				\\
				\hline
				\text{noisy circles} &0.165814882	&\textbf{1}	 & 0.156841194	& 0.189604897	&\textbf{1}
				\\
				\hline
				\text{noisy moons}	& 0.491529298 &\textbf{1}	& 0.50093108	&\textbf{1}	&\textbf{1}
				\\
				\hline
				\text{varied blob}	& 0.844780818	&0.906919484&	0.824744996	&\textbf{0.939402597} &	0.936238496	
				\\
				\bottomrule	
		\end{tabular}}
		\begin{tablenotes}
			\footnotesize
			\item{*} The best results are marked in \textbf{bold}.
		\end{tablenotes}
		\label{tab::1}
	\end{table*}

	Table \ref{tab::1} reports the ARI of our clustering algorithm and other clustering methods with the best parameter setting over four toy datasets. It can be evidently observed from the Table \ref{tab::1} that our algorithm has the best ARI performances on almost all data sets, further demonstrating the effectiveness of the algorithm. Our algorithm as well as DBSCAN recognizes the correct clusters on three data sets: {\tt aniso-blob}, {\tt noisy circles}, and {\tt noisy moons}.

	\subsection{Real Data Analysis}  \label{subsec::RealDataAnalysis}

	In our experiment, to assess the performance of various clustering methods, we evaluate the ARI among ADP-Cluster, DBSCAN, $k$-means and PDF-Cluster and our best-scored clustering forest on the following real data sets from UCI and Kaggle:
	\begin{itemize}
		\item 
		\textbf{Appendicitis}: The appendicitis data collected in the medical field was first put forward in \cite{Weiss1991Computer}.
		The whole data represents $7$ medical measures taken over $106$ patients on which the class label represents if the patient has appendicitis (class label $1$) or not (class label $0$).
		\item 
		\textbf{Customers}: The data set refers to clients of a wholesale distributor including the annual spending in monetary units on diverse product categories. This database available on UCI contains $440$ observations of dimension $8$ representing attributes such as {\tt fresh}, {\tt milk}, {\tt grocery}, {\tt frozen}, etc.
		\item 
		\textbf{Flee-beetles}: For three species of flea-beetles: concinna, heptapotamica, and heikertingeri, the whole data set was collected with six measurements: {\tt tars1}, {\tt tars2}, {\tt head}, {\tt aede1}, {\tt aede2} and {\tt aede3}. The whole data set consists of $74$ samples.
		\item 
		\textbf{Iris}:  Regarded as one of the best known database shown in the pattern recognition literature, 
		Iris contains $3$ classes of $50$ instance each, where each class refers to a type of iris plant. 
		The learning goal is to group the iris data with four features: {\tt sepal length}, {\tt sepal width}, {\tt petal length}, and {\tt petal width} into the true classes.
		\item 
		\textbf{Oliveoil}: Oliveoil comprises $572$ observations from oil analysis using measurements of different specimen for olive oil produced in various regions in Italy which can be further divided into three macro-areas: Centre-North, South, Sardinia. 
		This $8$-dimensional input data represent attributes such as {\tt palmitic}, {\tt palmitoleic}, {\tt stearic}, {\tt oleic}, {\tt linoleic}, {\tt linolenic}, {\tt arachidic}, {\tt eicosenoic}.
		The learning task is to reconstruct the macro-area membership.
	    \item 
\textbf{Wifi-localization}: The database comprising $2000$ observations was collected in indoor space by observing signal strengths of seven WiFi signals visible on a smartphone. The experiment was performed to explore how wifi signal strengths can be used to determine one of the indoor locations.
	    \item 
	    \textbf{Wine}: This data set including $178$ observations are the results of a chemical analysis of wines grown in the same region in Italy but derived from three different cultivars. The analysis determined the quantities of $13$ constituents found in each of the three types of wines.
%	    \item 
%	    \textbf{Newthyroid}: This data set is one of the several databases about Thyroid available at the UCI repository. It was collected with five attributes such as {\tt T3resin}, {\tt Thyroxin}, {\tt Triiodothyronine}, {\tt Thyroidstimulating}, {\tt TSH value}. The learning task is to group patients into three classes: 
%	    (i) normal, or 
%	    (ii) suffering from hyperthyroidism, or 
%	    (iii) hypothyroidism.
%	    \item 
%	    \textbf{Banknote}: This data set includes $1372$ observations extracted from images taken from genuine and forged banknote-like specimens by means of wavelet transform tools.
%	    The data has four attributes: the {\tt variance}, {\tt skewness}, {\tt kurtosis}, and 
%	    {\tt entropy} of wavelet transformed images.
%	    \item 
%\textbf{Wdbc}: This data set contains $30$ features computed from a digitized image of a fine needle aspirate (FNA) of a breast mass. They describe characteristics of the cell nuclei present in the image. The learning task is to determine if a found tumor is benign or malignant. 
	\end{itemize}

	\begin{table*}[h] 
		\setlength{\tabcolsep}{9pt}
		\centering
		\captionsetup{justification=centering}
		\caption{\footnotesize{Average ARI (adjusted rand index) over UCI Data Sets}}
		\resizebox{\textwidth}{13mm}{
			\begin{tabular}{c|ccccc}
				\toprule
				Datasets & ADP-Cluster	& DBSCAN &	$k$-means & PDF-Cluster	& Ours
				\\
				\hline 
				\hline
				\text{appendicitis}	& \textbf{0.525568993} &	0.419405321	&0.318301412	&0.468845316 &	0.518105812
				\\
				\hline
				\text{customers} &	0.260017768	&0.196905491&	0.400496777	& 0.012502467	&\textbf{0.625997132}\\
				\hline
				\text{flea} &0.879214629	& \textbf{1} & 	0.957543737	&\textbf{1} &	0.973650744\\
				\hline
				\text{iris}	& 0.568115942 &	0.61410887	& 0.716342113	&0.568115942 &	\textbf{0.778123403}
				\\
				\hline
				\text{oliveoil}	& 0.572903083	 & \textbf{1}	& 0.627921018	& 0.865827662	& \textbf{1}
				\\
				\hline
				\text{wifi localization} & \textbf{0.914080542} &0.869141447 & 0.314316461 &0.232882926 & 0.909102543
				\\
				\hline
				\text{wine} & 0.817666167 & 0.847096681 & 0.622913 & 0.845786696 & \textbf{0.872752411}
				\\
				\bottomrule	
		\end{tabular}}
		\begin{tablenotes}
			\footnotesize
			\item{*} The best results are marked in \textbf{bold}.
		\end{tablenotes}
		\label{tab::2}
	\end{table*}

	Table \ref{tab::2} summaries the ARI on the real data sets mentioned above. Careful observations will find that for most of these data sets,
	the best-scored forest clustering has significantly larger ARI than other density-based clustering methods. This superiority in cluster accuracy may be attributed to both the density estimation accuracy resulted from general architecture of random forest and the advantage of the density-based clustering method to group the data into arbitrarily shaped clusters. 
	We mention that interested readers can further
	tune the free parameters and 
	we believe that more accurate results could be obtained.

	\section{Proofs} \label{sec::Proofs}

	\begin{proof}[Proof of Theorem \ref{theorem::UncertaintyControlForest}]
		\textit{(i)}
		Let us first prove the inclusion $M_{\rho+\varepsilon}^{-2\sigma} \subset L_{D,\rho}$.
		To this end, we fix an $x \in M_{\rho+\varepsilon}^{-2\sigma}$, then we have $x \notin (\mathbb{R}^d \setminus M_{\rho+\varepsilon})^{+2\sigma}$, that is, for all $x' \in \mathbb{R}^d \setminus M_{\rho+\varepsilon}$, we have $\|x-x'\|_2>2\sigma$. In other words, if $x' \in \mathbb{R}^d$ satisfying $\|x-x'\|_2\leq 2\sigma$, then we have $x'\in M_{\rho+\varepsilon}$. 
        
        Now we show that for all $x_i\in D$, we have 
        \begin{align} \label{ChaoChaoIsAPig}
        f_{D,Z_{\mathrm{E}}}(x_i)\geq \rho 
        \quad
        \text { or } 
        \quad
        \|x-x_i\|_2> \sigma, 
        \end{align}
        whose proof will be conducted in the following by contradiction.       
		Suppose that there exists a sample $x_i 
		\in D$ with $f_{D,Z_{\mathrm{E}}}(x_i)<\rho$ and $\|x-x_i\|_2\leq \sigma$. 
		If we denote $A_{Z_t}(x_i)$ as the unique cell of the partition $A_{Z_t,\mathrm{P}}$ of the $t$-th tree in the forest where $x_i$ falls, 
		then the assumption $\mathrm{diam}(A_{Z_t}(x_i)) \leq \delta\leq \sigma$ implies that
		for any $x'\in A_{Z_t}(x_i)$, there holds
		$$
		\|x-x'\|_2
		\leq \|x-x_i\|_2+\|x_i-x'\|_2\leq 2\sigma
		$$ 
		and consequently we have $x'\in M_{\rho+\varepsilon}$, i.e., $A_{Z_t}(x_i) \subset M_{\rho+\varepsilon}$ for $t=1,\ldots,m$. This together with the normality of $\mathrm{P}$ yields
		$$
		\mu(A_{Z_t}(x_i)\setminus \{f \geq \rho+\varepsilon\})\leq \mu(M_{\rho+\varepsilon}\setminus \{f \geq \rho+\varepsilon\})=0,
		$$ 
		which leads to
		\begin{align*}
			\mathrm{P}(A_{Z_t}(x_i))
			& =\int_{A_{Z_t}(x_i)\cap \{f\geq \rho+\varepsilon\}} f d\mathrm{P}+\int_{A_{Z_t}(x_i)\setminus \{f\geq \rho+\varepsilon\}} f d\mathrm{P} 
			\\
			& =\int_{A_{Z_t}(x_i)\cap \{f \geq \rho+\varepsilon\}} f d\mathrm{P}\geq  \mu(A_{Z_t}(x_i))(\rho+\varepsilon).
		\end{align*}
		Consequently, we have 
		\begin{align} \label{UncertaintyReductio1}
			f_{\mathrm{P},Z_{\mathrm{E}}}(x_i)=\frac{1}{m}\sum^m_{t=1}\frac{\mathrm{P}(A_{Z_t}(x_i))}{\mu(A_{Z_t}(x_i))}\geq \rho+\varepsilon.
		\end{align}
		By $f_{D,Z_{\mathrm{E}}}(x_i)<\rho$ and $\|f_{D,Z_{\mathrm{E}}}(x_i)-f_{\mathrm{P},Z_{\mathrm{E}}}(x_i)\|_{\infty}\leq \varepsilon$, we find $f_{\mathrm{P},Z_{\mathrm{E}}}(x_i)< \rho+\varepsilon$, which contradicts (\ref{UncertaintyReductio1}). Therefore, for all $x_i\in D$, we have $f_{D,Z_{\mathrm{E}}}(x_i)\geq \rho$ or $\|x-x_i\|_2> \sigma$.
		
		Next, we show that there exist a sample $x_i \in D$ such that $\|x-x_i\|_2\leq \sigma$ by contradiction. If we denote $A_{Z_t}(x)$ as the unique cell of the partition $A_{Z_t,p}$ of the $t$-th tree in the forest where $x$ falls, then for all $x_i$, $i=1,\ldots,n$, we have $\|x-x_i\|_2>\sigma\geq \delta$, and consequently $A_{Z_t,p}\cap D=\emptyset$, $t=1,\ldots,m$. This leads to $f_{D,Z_{\mathrm{E}}}(x)=0$, which contradicts $f_{\mathrm{P},Z_{\mathrm{E}}}(x)\geq \rho+\varepsilon$ with the condition $\|f_{D,Z_{\mathrm{E}}}(x)-f_{\mathrm{P},Z_{\mathrm{E}}}(x)\|_{\infty}\leq \varepsilon$. Therefore, we conclude that there exists a sample $x_i \in D$ satisfying $\|x-x_i\|_2\leq \sigma$. This together with \eqref{ChaoChaoIsAPig} implies
	    $f_{D,Z_{\mathrm{E}}}(x_i)\geq \rho$, which means
	    $x \in L_{D,\rho}$. This finishes the proof of
		$M_{\rho+\varepsilon}^{-2\sigma} \subset L_{D,\rho}$.
		
		\textit{(ii)} To prove the second inclusion $L_{D,\rho} \subset M_{\rho-\varepsilon}^{+2\sigma}$,
         let us fix an $x \in L_{D,\rho}$, then there exists $x'\in D$ satisfying
		 $\|f_{D,Z_{\mathrm{E}}}(x')\|_2\geq \rho$ and $\|x-x'\|_2\leq \sigma$. 
		 Moreover, since $\|f_{D,Z_{\mathrm{E}}}(x)-f_{\mathrm{P},Z_{\mathrm{E}}}(x)\|_{\infty} \leq \varepsilon$, we have 
		 \begin{align} \label{XiaoYuIsAPigToo}
		 f_{\mathrm{P},Z_{\mathrm{E}}}(x')\geq \rho-\varepsilon. 
		 \end{align}
		 
		 Now we are able to prove the inclusion $L_{D,\rho} \subset M_{\rho-\varepsilon}^{+2\sigma}$
		 by contradiction.
		 Suppose that $x\notin M_{\rho-\varepsilon}^{+2\sigma}$. 
		 Since $\|x-x'\|_2\leq \sigma$,
		 then we have $B(x',\sigma)\subset \mathbb{R}^d\setminus M_{\rho-\varepsilon}$. 
		 If $A_{Z_t}(x')$ stands for the unique cell of the partition $A_{Z_t,p}$ of the $t$-th tree in the forest where $x'$ falls, since $\mathrm{diam}(A_{Z_t}(x'))\leq \delta\leq \sigma$, we thus have $A_{Z_t}(x')\subset B(x',\sigma)\subset \mathbb{R}^d\setminus M_{\rho-\varepsilon}$. This together with the normality of $\mathrm{P}$ yields
		 $$
		 \mu(A_{Z_t}(x')\setminus \{f < \rho-\varepsilon\})\leq \mu((X \setminus M_{\rho-\varepsilon})\setminus \{f > \rho-\varepsilon\})=0, 
		 $$
		 which leads to 
		\begin{align*}
			\mathrm{P}(A_{Z_t}(x'))
			& = \int_{A_{Z_t(x')}\cap \{h<\rho-\varepsilon\}} h \, d\mathrm{P} +\int_{A_{Z_t(x')} \setminus \{h<\rho-\varepsilon\}}h \, d\mathrm{P}  
			\\
			& = \int_{A_{Z_t(x')}\cap \{h<\rho-\varepsilon\}} h \, d\mathrm{P} < \mu(A_{Z_t(x')})(\rho-\varepsilon).
		\end{align*}
		Consequently we have
		\begin{align} \label{UncertaintyReductio2}
			f_{\mathrm{P}, Z_{\mathrm{E}}}(x')=\frac{1}{m}\sum^m_{t=1}\frac{\mathrm{P}(A_{Z_t}(x'))}{\mu(A_{Z_t}(x'))}< \rho-\varepsilon
		\end{align}
		which contradicts \eqref{XiaoYuIsAPigToo}.
        Therefore, we conclude that $x \in M_{\rho_\varepsilon}^{+2\sigma}$. This completes the proof of $L_{D,\rho} \subset M_{\rho-\varepsilon}^{+2\sigma}$.
	\end{proof}

	\begin{proof}[Proof of Theorem \ref{thr::BoundProbability}]
		The proof can be conducted by applying Theorem \ref{thr::LevelSetBound} directly and hence we need to verify its assumptions.
		
		Let us first prove that if 
		$\varepsilon^*\leq (\rho^{**}-\rho^*)/9$, $\delta \in (0,\delta_{\textit{thick}}]$, $\varepsilon\in (0,\varepsilon^*]$ and $\psi(\delta)<\tau$, then we have $\tau\leq \tau^*(\varepsilon^*)$.
		To this end, we define a set $E$ by
		\begin{align*}
		E := \{ \varepsilon' \in (0,\rho^{**}-\rho^*] : \tau^*(\varepsilon')\geq \tau \}.
		\end{align*}
		Obviously, we have $E \neq \emptyset$, since $\varepsilon^* < \infty$.
		This implies that there exists an $\varepsilon_0 \in E$ such that $\varepsilon_0 \leq \inf E+\varepsilon\leq \varepsilon^*$. Using the monotonicity of $\tau^*$, we conclude that $\tau \leq \tau^*(\varepsilon_0) \leq \tau^*(\varepsilon^*)$. 
		
		Next, we prove that for all $\rho > 0$,
		$(L_\rho)_{\rho\geq 0}$ satisfy \eqref{UncertaintyControl} 
		with probability not less than $1 - e^{-\varsigma}$.
		For $t = 1, \dots, m$,
		let the events $B_{1,t,\varepsilon}$ and $B_{2,t,\delta}$ be defined by
		\begin{align}
		B_{1,t,\varepsilon} & := \{ \|f_{\mathrm{D},Z_t} - f_{\mathrm{P},Z_t}\|_{\infty} \leq  \varepsilon \},
		\label{B1t}
		\\
		B_{2,t,\delta} & := \{ \forall A\in \mathcal{A}_{Z_t,p}:\mathrm{\mathrm{diam}}(A)\leq \delta \}.
		\label{B2t}
		\end{align}
		According to Proposition 15 and Inequality (19) in \cite{hang2018best},
		 there hold
		\begin{align*}
			\mathrm{P}(B_{1,t,\varepsilon}) \geq 1 - e^{-\varsigma} / (2m)
		\quad
		\text{ and }
		\quad
			\mathrm{P}(B_{2,t,\delta}) \geq 1 - e^{-\varsigma} / (2m)
		\end{align*}
		for all $t = 1, \dots, m$. Moreover,
		for the forest, we define the events $B_{1,\mathrm{E}}$ and $B_{2,\mathrm{E}}$ by
		\begin{align}
		B_{1,\mathrm{E},\varepsilon} & := \{ \| f_{\mathrm{D},Z_\mathrm{E}} - f_{\mathrm{P}, Z_\mathrm{E}} \|_{\infty} \leq  \varepsilon \},
		\label{B1E}
		\\
		B_{2,\mathrm{E},\delta} & := \{ \forall A \in \mathcal{A}_{Z_t,p} : \mathrm{diam}(A) \leq \delta, \, t=1,\dots,m \}.
		\label{B2E}
		\end{align}
		Since the splitting criteria $Z_1,\dots,Z_m$ are i.i.d.~from $\mathrm{P}_Z$, then we have
		\begin{align*}
			\mathrm{P}(B_{1,\mathrm{E},\varepsilon}) \geq 1 - e^{-\varsigma} / 2
			\quad
			\text{ and }
			\quad
			\mathrm{P}(B_{2,\mathrm{E},\delta}) \geq 1 - e^{-\varsigma} / 2
		\end{align*}
		and consequently we obtain
		\begin{align*}
			\mathrm{P}(B_{1,\mathrm{E},\varepsilon} \cap B_{2,\mathrm{E},\delta}) 
			= 1 - \mathrm{P}(B_{1,\mathrm{E},\varepsilon}^c \cup B_{2,\mathrm{E},\delta}^c) 
			\geq 1 - \mathrm{P}(B_{1,\mathrm{E},\varepsilon}^c) 
			          - \mathrm{P}(B_{2,\mathrm{E},\delta}^c) 
			\geq 1 - e^{-\varsigma}.
		\end{align*}
		This proves that for all $\rho > 0$,
		$(L_\rho)_{\rho\geq 0}$ satisfy \eqref{UncertaintyControl} 
		with probability not less than $1 - e^{-\varsigma}$
		and hence all the assumptions of Theorem \ref{thr::LevelSetBound} are indeed satisfied.
	\end{proof}

	To prove Theorem \ref{the::Consistency} concerning with the consistency of our clustering algorithm, we need the following technical lemma.

	\begin{lemma} \label{lem::trival}
		Let $(a_n)$, $(b_n)$ be strictly positive sequences and $\varsigma_n$ be the solution of equation
		\begin{align*}
			e^{2\varsigma_n}(\varsigma_n+a_n)=b_n.
		\end{align*}
		If $\lim_{n\to \infty} b_n=\infty$ and $\lim_{n\to \infty}a_n/b_n=0$, then $\lim_{n\to \infty} \varsigma_n=\infty$.
	\end{lemma}

	\begin{proof}[Proof of Lemma \ref{lem::trival}]
		We prove the lemma by contradiction.
		To this end, we assume that $\lim_{n\to \infty}\varsigma_n\neq \infty$. 
		Then there exists an $M > 0$, and a subsequence of $(\varsigma_n)$ denoted by $(\varsigma_{n_k})$ such that $|\varsigma_{n_k}| < M$ hold for all $k$.
		Consequently we obtain
		\begin{align*}
			b_{n_k}=e^{2\varsigma_{n_k}}(\varsigma_{n_k}+a_{n_k})<e^{2M}(M+a_{n_k})
		\end{align*}
		for all $k$. This together with the condition
		$\lim_{k\to \infty} b_{n_k}=\infty$ implies that $\lim_{k\to \infty} a_{n_k}=\infty$. 
		Therefore, we have
		\begin{align*}
			\varliminf_{k\to \infty}\frac{a_{n_k}}{b_{n_k}}
			\geq \varliminf_{k\to \infty} \frac{1}{e^{2M}} \cdot \frac{a_{n_k}}{M+a_{n_k}}
			= \frac{1}{e^{2M}},
		\end{align*}
		which contradicts the condition $\lim_{n\to\infty}a_n/b_n=0$
		and thus the assertion is proved.
	\end{proof}

	\begin{proof}[Proof of Theorem \ref{the::Consistency}]
		Let the events $B_{1,t,\varepsilon}$, 
		$B_{2,t,\delta}$,
		$B_{1,\mathrm{E},\varepsilon}$, and
		$B_{2,\mathrm{E},\delta}$
		be defined as in 
		\eqref{B1t},
		\eqref{B2t},
		\eqref{B1E}, and
		\eqref{B2E} respectively.
		According to Inequality (19) in \cite{hang2018best}, we have
		\begin{align*}
			\mathrm{P}^n ( B_{2,t,\delta_n} )
			\geq 1 - (K + 2 r) d \delta_n^{-1} p_n^{- c_T/(4d)}
		\end{align*}
		and consequently we obtain
		\begin{align*}
			\lim_{n\to \infty} \mathrm{P}^n (B_{2,t,\delta_n}) = 1.
		\end{align*}
		Since $m$ is finite and splitting criteria $Z_1,\dots,Z_m$ are i.i.d.~from $\mathrm{P}_Z$, we have
		\begin{align*}
			\lim_{n\to \infty} \mathrm{P}^n ( B_{2,\mathrm{E},\delta_n} ) = 1.
		\end{align*}
		Proposition 15 in \cite{hang2018best} shows that
		\begin{align*}
			\mathrm{P}^n ( B_{1,\mathrm{E},\varepsilon_n} )
			\geq 1 - 2 e^{-\varsigma_n}
		\end{align*}
		where $\varepsilon_n \geq \sqrt{\|f\|_{\infty} \mathcal{E}'_n} + \mathcal{E}'_n/3 + 2/n$ with
		\begin{align} \label{eq::E_n}
			\mathcal{E}'_n := 
			8 n^{-1} \mu(B_r)^{-1} e^{2 \varsigma_n} p_n^{2a} 
			\bigl( (8 d + 1) \varsigma_n + 23 \log n + 8 a d \log p_n \bigr).
		\end{align}
		Obviously, there exists certain $q \in (0, 1)$ such that
		$\varepsilon'_n := \sqrt{\|f\|_{\infty}\mathcal{E}'_n}+\mathcal{E}'_n/3 \geq (1 - q) \varepsilon_n$. 
		
		Next, 
		with the help of Lemma \ref{lem::trival},
		we show that if $\varepsilon'_n \to 0$, then we have $\varsigma_n \to \infty$ with $\varsigma_n$ satisfying \eqref{eq::E_n}.
		Clearly, we have
		\begin{align*} 
			\mathcal{E}'_n = 9 \bigl( \sqrt{\|f\|_{\infty} + 4 \varepsilon'_n/3} - \sqrt{\|f\|_{\infty}} \bigr)^2 / 4.
		\end{align*}
		Plugging this into \eqref{eq::E_n}, we obtain
		\begin{align*}
			e^{2\varsigma_n}((8d+1)\varsigma_n+23\log n+8ad \log p_n)=
			9 n p_n^{- 2a} \mu(B_r) 
			\bigl( \sqrt{\|f\|_{\infty}+ 4 \varepsilon'_n / 3} - \sqrt{\|f\|_{\infty}} \bigr)^2 / 32.
		\end{align*}
	    Now, by setting
		\begin{align*}
		    			a_n & := (23 \log n + 8 a d \log p_n) / (8 d  + 1),
		    			\\
			b_n& := 9 n p_n^{-2a} \mu(B_r) \bigl( \sqrt{\|f\|_{\infty} + 
				4 \varepsilon'_n / 3} - \sqrt{\|f\|_{\infty}} \bigr)^2 
			/ (32(8d + 1)),
		\end{align*}
		it can be easily verified that
		there exist finite constants $c_1$, $c_2$, $c_3$, and $c_4$ such that
		\begin{align*}
			\lim_{n\to \infty} b_n 
			= \lim_{n\to \infty} c_1 n p_n^{-2a} {\varepsilon'_n}^2
			\geq \lim_{n\to \infty} c_2 n p_n^{-2a} {\varepsilon_n}^2
			= \infty
		\end{align*}
		and
		\begin{align*}
			\lim_{n\to \infty} a_n / b_n
			=\lim_{n\to \infty} c_3 (n^{-1} \log n) p_n^{2a} {\varepsilon'_n}^{-2}
			\leq \lim_{n\to \infty} c_4 (n^{-1} \log n) p_n^{2a} {\varepsilon'_n}^{-2}
			= 0.
		\end{align*}
	    Then, Lemma \ref{lem::trival} with the above $a_n$ and $b_n$
	    implies $\varsigma_n \to \infty$
		and consequently we have
		\begin{align*}
			\lim_{n\to \infty} \mathrm{P}^n ( B_{1,t,\varepsilon_n} ) = 1.
		\end{align*}
		Since $m$ is finite and splitting criteria $Z_1,\dots,Z_m$ are i.i.d.~from $\mathrm{P}_Z$, we have
		\begin{align*}
			\lim_{n\to \infty} \mathrm{P}^n ( B_{1,\mathrm{E},\varepsilon_n} ) = 1
		\end{align*}
		which completes the proof of consistency
		according to Theorem \ref{thr::LevelSetBound} and Section A.9 in \cite{steinwart2015suppA}.
	\end{proof}

	\begin{proof}[Proof of Corollary \ref{col::rho_rates}]
		For $n\geq 1$, 
		we define 
		$$
		\varepsilon_n^*:=\varepsilon_n+(\tau_n/\underline{c}_{\textit{sep}})^\kappa.
		$$ 
		Since sequences $(\varepsilon_n)$, $(\delta_n)$ and $(\tau_n)$ converge to $0$, we have $\delta_n \in (0,\delta_{\textit{thick}}]$ and 
		for all sufficiently large $n$, there holds
		$$
		\varepsilon_n^*\leq (\rho^{**}-\rho^*)/9.
		$$
		Moreover, the assumed $\tau_n$ and $\sigma_n$ satisfy
		\begin{align*}
			\lim_{n\to\infty} \tau_n/ (2\sigma_n)^{\gamma}
			= \lim_{n\to\infty} 2^{-\gamma} 
			   \bigl( n^{\lambda(1-\alpha)}
			           (\log n)^{\lambda(\alpha-1)}
			           ( \log\log n / \log \log \log n ) \bigr)^{\gamma/2(\gamma\kappa+\epsilon)}
			 = \infty
		\end{align*}
		and therefore
		we have 
		$$
		\tau_n > 3c_{\textit{thick}} (2\sigma_n)^\gamma=\psi(2\sigma_n)
		$$ 
		for all sufficiently large $n$. Set
		\begin{align*}
			\varsigma_n : = \log\log n, 
			\qquad 
			p_n := \bigl( n/\log n \bigr)^{\frac{2d}{c_T\alpha+4ad}},
		\end{align*}
		and denote $\mathcal{E}_{\varsigma_n,p}$ as in \eqref{eq::CalE_n}.
		Since $\mathcal{E}_{\varsigma_n,p}\to \infty$ and $\mathcal{E}_{\varsigma_n,p}/n\to 0 $, we have
		\begin{align*}
			\sqrt{\|f\|_{\infty} \mathcal{E}_{\varsigma_n, p}/ n} 
			+ \mathcal{E}_{\varsigma_n, p} / (3n) 
			+ 2/n
			& \sim \sqrt{\|f\|_{\infty} \mathcal{E}_{\varsigma_n, p}/ n} 
			\\
			&\lesssim \bigl( 
			  n^{-1} (\log n)^3 (n/\log n)^{4 a d \lambda / c_T} \bigr)^{1/2}
			  \\
			&\lesssim \bigl( n^{- \lambda \alpha} (\log n)^{2+\lambda \alpha} \bigr)^{1/2}.
		\end{align*}
		Consequently, for all sufficiently large $n$, we have
		\begin{align*}
			\varepsilon_n  
			& = \bigl( n^{- \lambda \alpha} (\log n)^{2+\lambda \alpha} \log\log n \bigr)^{\gamma\kappa/(2(\gamma\kappa+\epsilon))}
			\\
			& \geq \bigl( n^{\lambda \alpha} (\log n)^{2+\lambda \alpha} \log\log n \bigr)^{1/2}
			\\
			& > \sqrt{\|f\|_{\infty} \mathcal{E}_{\varsigma_n, p}/ n} 
			+ \mathcal{E}_{\varsigma_n, p} / (3n) 
			+ 2/n
		\end{align*}
		and therefore condition \eqref{BoundVarepsilon} on $\varepsilon_n$ is satisfied.
		Moreover, there holds
		\begin{align*}
			p_n^{-1} \bigl( 2 m d (K + 2 r) e^{\varsigma_n} / \delta_n \bigr)^{4 d / c_T}
			& \lesssim \bigl(
			      n^{\lambda/2} (\log n)^{-\lambda} (\log \log n)^{-1/2} \bigr)^{4 d / c_T} 
			      (\log n / n)^{2 d \lambda / c_T}
			\\
			& = (\log\log n)^{- 2d / c_T}
		\end{align*}
		and consequently we have 
		\begin{align*}
			\lim_{n \to \infty} 
			p_n^{-1} \bigl( 2 m d (K+2r) e^{\varsigma_n} / \delta_n \bigr)^{4d/c_T} = 0.
		\end{align*}
		In other words, for all sufficiently large $n$, there holds
		\begin{align*}
			p_n > \bigl( 2 m d (K+2r) e^{\varsigma_n} / \delta_n \bigr)^{4d/c_T}
		\end{align*}
		and therefore condition \eqref{BoundSplits} on $p_n$ is satisfied.
		
		Now, by applying Theorem \ref{the::rho_rates}, there exist an $n_0 \geq 1$ and a constant $\overline{c}$ such that the right-hand side of inequalities \eqref{RhoExessError} holds for $n \geq n_0$. Moreover, if $\kappa$ is exact, \eqref{equ::tau_rho} holds for all $n\geq n_0$.
	\end{proof}

	\begin{proof}[Proof of Corollary \ref{col::sets_rate}]
		Similar as the proof of Theorem \ref{the::rho_rates}, we prove that for all sufficiently large $n$, there holds
		\begin{align*}
		\tau_n & \geq \psi(2\sigma_n),
		\\
	    \varepsilon_n & > \sqrt{\|f\|_{\infty} \mathcal{E}_{\varsigma_n, p}/ n} 
	    + \mathcal{E}_{\varsigma_n, p} / (3n) 
	    + 2/n,
			\\
		p_n & > \bigl( 2md(K+2r)e^{\varsigma_n} / \delta_n \bigr)^{4d/c_T},
		\end{align*}
		with $\mathcal{E}_{\varsigma_n,p}$ as in \eqref{eq::CalE_n}, and thus
		the conditions in Theorem \ref{the::sets_rate} are all satisfied.
		Then, for such $n$, by applying Theorem \ref{the::sets_rate}, we obtain
		\begin{align*}
			\mathrm{P}^n \biggl( D : \sum^2_{i=1}
			\mu(B_i(D)\triangle A_i^*)
			\leq \bigl( 
			      7 c_{\textit{flat}} \varepsilon_n
			      + c_{\textit{flat}} (\tau_n/\underline{c}_{\textit{sep}})^\kappa
			      \bigr)^{\vartheta}
			      + 6 c_{\textit{bound}}\delta_n^{\alpha_0} 
			\biggr)
			\geq 1 - 1/\log n.
		\end{align*}
		Elementary calculations show that with the assumed $\varepsilon_n$,$\tau_n$,  and $\delta_n$, there exists a constant $c > 0$ such that
		\begin{align*}
			\mathrm{P}^n \biggl( D : \sum^2_{i=1} \mu(B_i(D) \triangle A_i^*) 
			\leq c \bigl( n^{- \lambda \alpha} (\log n)^{2+\lambda \alpha} \log \log n \bigr)^{\frac{\varrho\vartheta}{2(\varrho+\vartheta)}} \biggr)
			\geq 1 - 1 / \log n.
		\end{align*}
		Obviously, we have
		\begin{align*}
			\sum^2_{i=1} \mu(B_i(D) \triangle A_i^*)
			\leq 2\mu(\mathcal{X})<\infty.
		\end{align*}
		Therefore, we can choose a constant $c$ large enough such that the desired inequality holds for all $n \geq 1$.
	\end{proof}

	\section{Conclusion} \label{sec::Conclusion}

	In this paper, we present an algorithm called \emph{best-scored clustering forest} to efficiently solve the single-level density-based clustering problem. From the theoretical perspective, our main results comprise statements and complete analysis of statistical properties such as consistency and learning rates. The convergence analysis is conducted within the framework established in \cite{steinwart2015adaptive}. With the help of best-scored random forest density estimator proposed by \cite{hang2018best}, we show that consistency 
	of our proposed clustering algorithm
	can be established 
	with properly chosen hyperparameters
	of the density estimators and partition diameters. Moreover, we obtain fast rates of convergence for estimating the clusters under certain mild conditions on the underlying density functions and target clusters.
	Last but not least, the excellence of \emph{best-scored clustering forest} was demonstrated by various numerical experiments. 
	On the one hand, 
		the new approach provides better average adjusted rand index (ARI) than other state-of-the-art methods such as ADP-Cluster, DBSCAN, $k$-means and PDF-Cluster on synthetic data, while providing average ARI that are at least comparable on several benchmark real data sets.
	On the other hand, 
	due to the intrinsic advantage of random forest, it is  to be expected that this strategy enjoys satisfactory computational efficiency by taking utmost advantage of the parallel computing.

	\bibliographystyle{chicago}
%	\bibliography{HANGBib}
	
\end{document}